%% file: main.tex
\documentclass{article}
\usepackage{arxiv}

\input{common_packages.tex}
\usepackage{wrapfig}
\usepackage{url}
\usepackage{sidecap}
\usepackage{natbib}
\input{common_commands.tex}

\input{common_colors.tex}
\usepackage{enumitem}

\newlist{todolist}{itemize}{2}
\setlist[todolist]{label=$\square$}

\DeclareMathOperator*{\emb}{emb}
\usepackage{fonts}

\usepackage{xparse}
\usepackage{tikz}

\newcommand{\lip}{\operatornamewithlimits{Lip}}

\makeatletter
\renewcommand{\boxed}[1]{\text{\fboxsep=.2em\fbox{\m@th$\displaystyle#1$}}}
\makeatother

\title{Universal Joint Approximation of Manifolds and Densities by Simple Injective Flows}
\author{
  Michael Puthawala \\
  Rice University \\
  \texttt{map19@rice.edu} \\
  \And
  Matti Lassas \\
  University of Helsinki \\
  \texttt{matti.lassas@helsinki.fi} \\
  \And
  Ivan Dokmani\'c \\
  University of Basel \\
  \texttt{ivan.dokmanic@unibas.ch} \\
  \And
  Maarten de Hoop \\
  Rice University \\
  \texttt{mdehoop@rice.edu} \\
}

\begin{document}

\maketitle

\begin{abstract}
    We study approximation of probability measures supported on $n$-dimensional manifolds embedded in $\R^m$ by \emph{injective flows}---neural networks composed of invertible flows and injective layers. We show that in general, injective flows between $\R^n$ and $\R^m$ universally approximate measures supported on images of \emph{extendable} embeddings, which are a subset of standard embeddings: when the embedding dimension $m$ is small, topological obstructions may preclude certain manifolds as admissible targets. When the embedding dimension is sufficiently large, $m \geq 3n+1$, we use an argument from algebraic topology known as the clean trick to prove that the topological obstructions vanish and injective flows universally approximate any differentiable embedding. Along the way we show that the studied injective flows admit efficient projections on the range, and that their optimality can be established "in reverse," resolving a conjecture made in \citep{brehmer2020flows}
\end{abstract}

\section{Introduction}

Invertible flow networks emerged as powerful deep learning models to learn maps between distributions \citep{durkan2019cubic,grathwohl2018ffjord,huang2018neural,jaini2019sum,kingma2016improving,kingma2018glow,kobyzev2020normalizing,kruse2019hint,papamakarios2019normalizing}. They generate high-quality samples \citep{kingma2018glow} and facilitate solving scientific inference problems \citep{brehmer2020flows,kruse2021benchmarking}. 

By design, however, invertible flows  are bijective and may not be a natural choice when the target distribution has low-dimensional support. This problem can be overcome by combining bijective flows with expansive, injective layers, which map to higher dimensions \citep{brehmer2020flows,cunningham2020normalizing,kothari2021trumpets}. Despite their empirical success, the theoretical aspects of such globally injective architectures are not well understood.






In this work, we address approximation-theoretic properties of injective flows. We prove that under mild conditions these networks universally approximate probability measures supported on low-dimensional manifolds and describe how their design enables applications to inference and inverse problems.

\subsection{Prior Work}
The idea to combine invertible (coupling) layers with expansive layers has been explored by \citep{brehmer2020flows} and \citep{kothari2021trumpets}.
\citet{brehmer2020flows} combine two flow networks with a simple expansive element (in the sense made precise in Section \ref{sec:expans}) and obtain a network that parameterizes probability distributions supported on manifolds.\footnote{More precisely, distributions on manifolds are parameterized by the pushforward (via their network) of a simple probability measure in the latent space.}

\citet{kothari2021trumpets} propose expansive coupling layers and build networks similar to that of \citet{brehmer2020flows} but with an arbitrary number of expressive and expansive elements. They observe that the resulting network trains very fast with a small memory footprint, while producing high-quality samples on a variety of benchmark datasets. 

While (to the best of our knowledge) there are no approximation-theoretic results for injective flows, there exists a body of work on universality of invertible flows; see \citet{kobyzev2020normalizing} for an overview. Several works show that certain bijective flow architectures are distributionally universal. This was proved for autoregressive flows with sigmoidal activations by  \citet{huang2018neural} and for sum-of-squares polynomial flows \citep{jaini2019sum}. \citet{teshima2020coupling} show that several flow networks including those from \citet{huang2018neural} and \citet{jaini2019sum} are also universal approximators of diffeomorphisms.

The injective flows considered here have key applications in inference and inverse problems; for an overview of deep learning approaches to  inverse problems, see \citep{arridge2019solving}.
\citet{bora2017compressed} proposed to regularize compressed sensing problems by constraining the recovery to the range of (pre-trained) generative models. Injective flows with efficient inverses as generative models give an efficient algorithmic projection\footnote{Idempotent but in general not orthogonal.} on the range, which facilitates implementation of reconstruction algorithms.
An alternative approach is Bayesian, where flows are used to obtain tractable variational approximations of posterior distributions over parameters of interest, via supervised training on labeled input-output data pairs.
\citet{ardizzone2018analyzing} encode the dimension-reducing forward process by an invertible neural network (INN), with additional outputs used to encode posterior variability. Invertibility guarantees that a model of the  inverse process is learned implicitly. For a given measurement, the inverse pass of the INN approximates the posterior over parameters. 
\citet{sun2020deep} propose variational approximations of the posterior using an untrained deep generative model. They train a normalizing flow which produces samples from the posterior, with the prior and the noise model given implicitly by the regularized misfit functional. In \citet{kothari2021trumpets} this procedure is adapted to priors specified by injective flows which yields significant improvements in computational efficiency.

\vspace{-2.5mm}
\subsection{Our Contribution}

We derive new approximation results for neural networks composed of bijective flows and injective expansive layers, including those introduced by \citep{brehmer2020flows} and \citep{kothari2021trumpets}. We show that these networks universally jointly approximate a large class of manifolds and densities supported on them.

We build on the results of \citet{teshima2020coupling} and develop a new theoretical device which we refer to as the \textit{embedding gap}. This gap is a measure of how nearly a mapping from $\Rea^o \to \Rea^m$ embeds an $n$-dimensional manifold in $\Rea^m$, where $n \leq o$. We find a natural relationship between the embedding gap and the problem of approximating probability measures with low-dimensional support. 

We then relate the embedding gap to a relaxation of universality we call the \textit{manifold embedding property}. We show that this property captures the essential geometric aspects of universality and uncover important topological restrictions on the approximation power of these networks, to our knowledge, heretofore unknown in the literature. We give an example of an absolutely continuous measure $\mu$ and embedding $f \colon \Rea^2 \to \Rea^3$ such that $\pushf{f}{\mu}$ can not be approximated with combinations of flow layers and linear expansive layers. This may be surprising since it was previously conjectured that networks such as those of \citet{brehmer2020flows} can approximate any ``nice'' density supported on a ``nice'' manifold. We establish universality for manifolds with suitable topology, described in terms of \textit{extendable embeddings}. We find that the set of extendable embeddings is a proper subset of all embeddings, but when $m \geq 3n + 1$, via an application of the \textit{clean trick} from algebraic topology, we show that all diffeomorphisms are extendable and thus injective flows approximate distributions on arbitrary manifolds. Our universality proof also implies that optimality of the approximating network can be established in reverse: optimality of a given layer can be established without optimality of preceding layers. This settles a (generalization of a) conjecture posed for a three-part network (composed of two flow networks and zero padding) 
in \citep{brehmer2020flows}.
Finally, we show that these universal architectures are also practical and admit exact layer-wise projections, as well as other properties discussed in Section \ref{sec:additional-properties}.

\section{Architectures Considered}
\label{sec:desc-of-arch}

Let $C(X,Y)$ denote the space of continuous functions $X \to Y$. Our goal is to make statements about networks in $\cF \subset C(X, Y)$ that are of the form:
\begin{align}
    \label{eqn:network-def}
    \cF = \cT^{n_{L}}_{L} \circ \cR_{L}^{n_{L-1},n_{L}} \circ \dots \circ \cT_1^{n_{1}} \circ \cR^{n_0,n_1}_1 \circ \cT^{n_{0}}_0
\end{align}
where $\cR^{n_{\ell-1},n_\ell}_{\ell} \subset C(\Rea^{n_{\ell-1}}, \Rea^{n_{\ell}})$,  $\cT^{n_\ell}_{\ell} \subset C(\Rea^{n_{\ell}},\Rea^{n_{\ell}}) \nonumber$, $L \in \Nat$, $n_0 = n$, $n_{L} = m$, and $n_{\ell} \geq n_{\ell-1}$ for $\ell = 1,\dots,L$. We introduce a well-tuned shorthand notation and write $\cH \circ \cG \coloneqq \set{h \circ g\colon h \in \cH, g \in \cG}$ throughout the paper.

We identify $\cR$ with the expansive layers and $\cT$ with the bijective flows. 
Loosely speaking, the purpose of the expansive layers is to allow the network to parameterize high-dimensional functions by low-dimensional coordinates in an injective way. The flow networks give the network the expressivity necessary for universal approximation of manifold-supported distributions. 

\subsection{Expansive Layers}
\label{sec:expans}

The expansive elements  transform an $n$-dimensional manifold $\cM$ embedded in $\Rea^{n_{\ell-1}}$, and embed it in a higher dimensional space $\Rea^{n_{\ell}}$. To preserve the topology of the manifold they are injective. We thus make the following assumptions about the expansive elements:

\begin{definition}[Expansive Element]
    A family of functions $\cR \subset C(\R^n,\R^m)$ is called an family of expansive elements if $m > n$, and each $R \in \cR$ is both injective and Lipschitz.
\end{definition}
Examples of expansive elements include
\begin{enumerate}
    \item[(R1)] Zero padding: $R(x) = \bmat{x^T,\bszero^{(m-n)}}^T$ where $\bszero^{(m-n)}$ is the zero vector \citep{brehmer2020flows}.
        

    \item[(R2)] Multiplication by an arbitrary full-rank matrix, or one-by-one convolution:
    \begin{align}
        R(x) = W x, \quad \text{or}\quad R(x) = w \star x
    \end{align}
    where $W \in \Rea^{m\times n}$ and $\rank(W) = n$ \citep{cunningham2020normalizing}, and $w$ is a convolution kernel $\star$ denotes convolution \citep{kingma2018glow}.
    \item[(R3)] Injective $\relu$ layers: $R{(x)} = \relu(W x)$, $W = \bmat{B^T,-D B^T ,M^T}^T$, {or }$R(x) = \relu\paren{\bmat{w^T, -w^T} \star x}$
    for matrix $B  \in \GL_n(\Rea)$, positive diagonal matrix $D  \in \Rea^{n \times n}$, and arbitrary matrix $M \in \Rea^{(m-2n)\times n}$  \citep{puthawala2020globally}.
    \item[(R4)] Injective $\relu$ networks \citep[Theorem 5]{puthawala2020globally}. These are functions  $R:\R^n\to \R^m$ of the form $R(x) = W_{L+1}\relu( \dots \relu( W_1 x + b_1)\dots ) + b_L$ where $W_\ell$ are $n_{\ell+1} \times n_{\ell}$ matrices and $b_{\ell}$ are the bias vectors in $\Rea^{n_{\ell+1}}$. The weight matrices $W_L$ satisfy the Directed Spanning Set (DSS) condition for $\ell\leq L$ (that make all layers injective) and $W_{L+1}$ is a generic matrix which makes the map $R:\R^n\to \R^m$ injective {where $m \geq 2n+1$}. Note that the DSS condition requires that $n_{\ell }\geq 2n_{\ell-1}$ for $\ell\le L$ and we have $n_1=n$ and $n_{L+1}=m$.
\end{enumerate}

{Continuous piecewise-differentiable functions with bounded gradients are always Lipschitz. Thus, the Lipschitzness assumption is automatically satisfied by feed-forward networks with piecewise-differentiable activation functions with bounded gradients. This includes compositions of $\relu$ and sigmoid layers.}

\subsection{Bijective Flow Networks}
\label{sec:expres}
The bulk of our theoretical analysis is devoted to the bijective flow networks, which bend the range of the expansive elements into the correct shape. We make the following assumptions about the expressive elements:
\begin{definition}[Bijective Flow Network]
    \label{def:express-element}
    Let $\cT \subset C(\Rea^{n}, \Rea^{n})$ for $n \in \Nat$. We call $\cT$ a family of bijective flow networks if every $T \in \cT$ is Lipschitz and bijective. 
\end{definition}

Examples of bijective flow networks include

\begin{enumerate}
    \item[(T1)] \textit{Coupling flows}, introduced by \citep{dinh2014nice} consider $R(\bx) = H_{k} \circ \dots \circ H_{1}(\bx)$ where
    \begin{align}
        \label{eqn:coupling-flow}
        H_{i}(\bx) = \bmat{h_{i}\paren{\bracketed{\bx}_{1:d},g_{i}\paren{\bracketed{\bx}_{d+1:n}}} \\ \bracketed{\bx}_{d+1:n}}.
    \end{align}
    In Eqn. \ref{eqn:coupling-flow}, $h_{i}\colon \Rea^{d} \times \Rea^{e} \rightarrow \Rea^d$ is invertible w.r.t. the first argument given the second, and $g_{i} \colon \Rea^{n - d} \rightarrow \Rea^e$ is arbitrary. Typically in practice the operation in Eqn. \ref{eqn:coupling-flow} is combined with additional invertible operations such as permutations, masking or convolutions \citep{dinh2014nice,dinh2016density,kingma2018glow}.
    \item[(T2)] \textit{Autoregressive flows}, introduced by \citet{kingma2016improving} are generalizations of triangular flows {$A \colon \Rea^n \to \Rea^n$ where for $i = 1,\dots,n$ the $i$'th value of $A$ is given by } of the form 
    \begin{align}
        \label{eqn:autoregressive-def}
        \bracketed{A}_i(\bx) = h_{i}\paren{\bracketed{\bx}_i, g_{i}\paren{\bracketed{\bx}_{1:i-1}}}
    \end{align}
     In Eqn. \ref{eqn:autoregressive-def}, $h_{i}\colon \Rea \times \Rea^m \rightarrow \Rea$ where again $h_{i}$ is invertible w.r.t. the first argument given the second, and $g_{i}\colon \Rea^{i-1} \rightarrow \Rea^m$ is arbitrary {except for $g_1 = \bszero$}. In \citet{huang2018neural}, the authors choose $h_{i}(x,\by)$, where $\by \in \Rea^m$, to be a multi-layer perceptron (MLP) of the form
     \begin{align}
         h_{i}(x,\by) = \phi \circ W_{p,\by}\circ\dots\circ \phi \circ W_{1,\by}(x)
     \end{align}
     where $\phi$ is a sigmoidal increasing non-linear activation function.
\end{enumerate}

\section{Main Results}
\label{sec:main-results}

\subsection{Embedding Gap}
\label{sec:main-results:preliminaries}


We call a function $f$ an embedding and denote it by  $f\in \emb(X,Y)$ if $f : X \to Y$ is continuous, injective, and $f^{-1}\colon f(X) \to X$ is continuous\footnote{Note that if $X$ is a compact set, then continuity of the of $f^{-1}\colon f(X) \to X$ is automatic, and need not be assumed \citep[Cor.\  13.27]{MR2548039}. Moreover, if $f:\R^n\to \R^m$ is a continuous injective map that satisfies $|f(x)|\to \infty$ as $|x|\to \infty$, then by
\citep[Cor.\ 2.1.23]{Mukherjee} the map $f^{-1}\colon f(\R^n) \to \R^n$ is continuous.}. { Also 
we denote by $ \emb^k(\R^n,\R^m)$ the set of maps $f\in \emb(\R^n,\R^m)\cap  C^k(\R^n,\R^m)$ which differential $df|_x:\R^n\to \R^m$ is injective at all points $x\in \R^n$.}
We now introduce the \textit{embedding gap}, a non-symmetric notion of distance between $f$ and $g$. This quantifies the degree to which a mapping $g \in \emb(\Rea^o,\Rea^m)$ fails to embed a manifold $\cM = f(K)$ for compact $K\subset \Rea^n$ where $f \in \emb(K,\Rea^m)$. Later in the paper, $f$ will be the function to be approximated, and $g$ an approximating flow-network. 

\begin{definition}[Embedding Gap]
    \label{def:b-k}
    Let {$n \leq p \leq o \leq m$,} $K \subset \Rea^n$ be compact and non-empty, $W \subset \Rea^o$ be compact and contain the closure of set $U$ which is open in the subspace topology of some vector subspace $V$ of dimension $p$, where $f \in \emb(K,\Rea^m)$ and $g\in \emb(W, \Rea^m)$. The Embedding Gap between $f$ and $g$ on $K$ and $W$ is
    \begin{align}
            \label{eqn:b-k-def}
            B_{K,W}(f,g) = \inf_{r \in \emb(f(K),g(W))}
            \norm{I - r}_{L^\infty(f(K))}
    \end{align}
    where $I\colon f(K) \to f(K)$ is the identity function and $\norm{h}_{L^\infty(X)} = \esssup_{x \in X} \norm{h{(x)}}_2$ for $h\colon X \to Y${, where $Y$ is some $L^\infty$ space}. We refer to the embedding gap between $f$ and $g$ without specifying $K$ and $W$ when it is clear from context.
\end{definition}

\begin{remark}
    As $W\subset \R^o$ contains  {$U$, }an open set {in $V$}, there is an affine map $A:\R^n\to V$ such that $A(K)\subset W$. Thus, the map $r_0=g\circ A\circ f^{-1}:f(K)\to g(W)$ is an injective continuous map from a compact set to its range and hence $r_0\in \emb(f(K),g(W))$. This proves that the infimum in \ref{eqn:b-k-def} is non-empty.
\end{remark}

Before giving properties of $B_{K,W}(f,g)$, we briefly describe its interpretation and meaning. We denote by $\cP(X)$ the set of probability measures over $X$. If the embedding gap between $f$ and $g$ is small, then $g^{-1}\circ r$ embeds the range of $f$ for an $r$ that is nearly the identity. Hence $g^{-1}$ nearly embeds the range of $f$ into $\R^o$. $B_{K,W}(f,g)$ also serves as an upper bound 
\begin{align*}
    \label{eqn:wasserstein-bound:1}
    \inf_{\mu_o \in \cP(W)}\  \wasstwo{\pushf{f}{\mu_n}}{\pushf{g}{\mu_o}} \leq B_{K,W}(f,g)
\end{align*}
where $\mu_n \in \cP(K)$ is given, and $\wasstwo{\nu_1}{\nu_2}$ denotes the Wasserstein-2 distance with $\ell^2$ ground metric \citep{villani2008optimal}. This is proven in Lemma \ref{lem:b-k-helper} part 9. The above result has a simple meaning in the context of machine learning. Suppose we want to learn a generative model $g$ to (approximately) sample from a probability measure $\nu$ with low-dimensional support, by applying $g$ to samples from a \textit{base} distribution $\mu_o$. Suppose further that $\nu$ is a pushforward of some (known or unknown) distribution $\mu_n$ via $f$. The embedding gap $B_{K, W}(f, g)$ then upper bounds the 2-Wasserstein distance between $\nu$ and $g_\# \mu_0$ for the best possible choice of $\mu_o$.\footnote{The choice of $p$-Wasserstein distance is suitable for measures with mismatched low-dimensional support; this has been widely exploited in training generative models \citep{arjovsky2017wasserstein}.}



In the context of optimal transport, the embedding $r$ can be interpreted as a candidate transport map from any measure pushed forward by $f$, that can be pulled back through $g$. Loosely speaking, for $\mu_o' = \pushf{g^{-1}\circ r \circ f}{\mu_n}$, $r$ transports $\pushf{f}{\mu_n}$ to $\pushf{g}{\mu'_o}$ with cost no more than $\norm{I - r}_{L^\infty(f(K))}$. See Fig. \ref{fig:b-k-visualization} for a visualization of the embedding gap between two toy functions. The embedding gap satisfies inequalities useful for studying networks of the form of Eqn. \ref{eqn:network-def}, see Lemma \ref{lem:b-k-helper}.

{In the remainder of this section we use the embedding gap to prove universality of neural networks. The set $f(K)$ will be a target manifold to approximate, and $g$ will be a neural network of the form \eqref{eqn:network-def}. The embedding gap requires $g$ to be a proper embedding and so, in particular, injective. This is why we require injectivity of both the expansive and bijective flow layers.}

\begin{figure}
    \centering
    \begin{subfigure}{.30\linewidth}
        \centering
        \includegraphics[width=\linewidth]{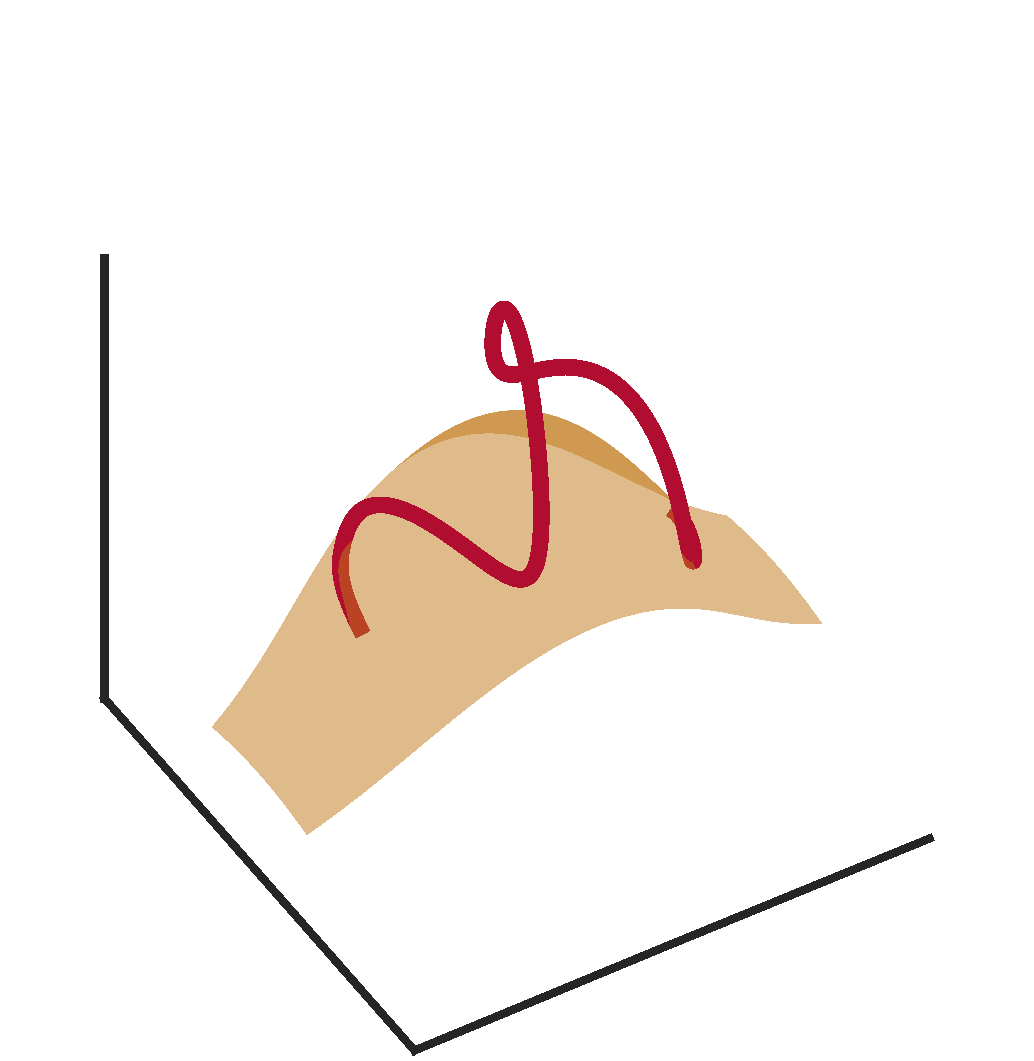}
        \label{fig:b-k-visualization:step-1}
    \end{subfigure}
    \begin{subfigure}{.30\linewidth}
        \centering
        \includegraphics[width=\linewidth]{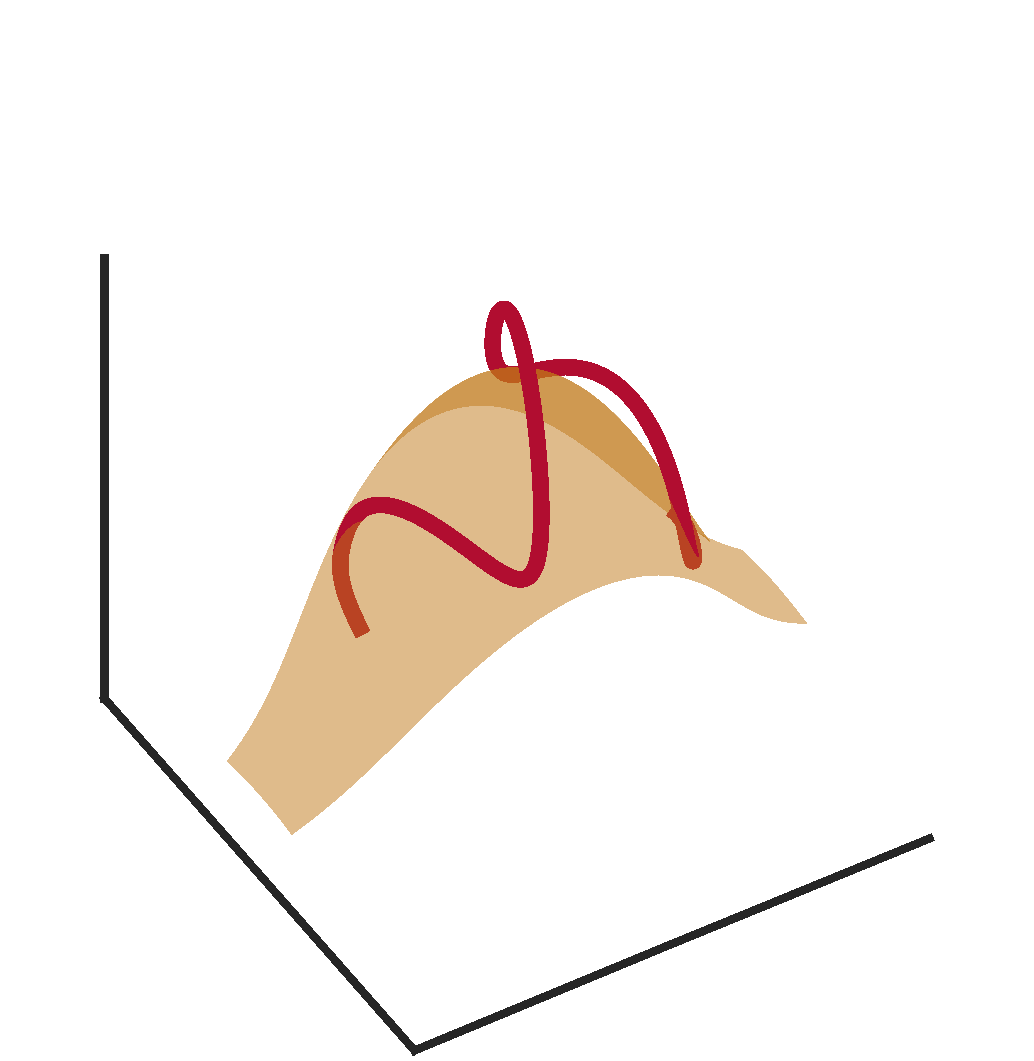}
        \label{fig:b-k-visualization:step-2}
    \end{subfigure}
    \begin{subfigure}{.30\linewidth}
        \centering
        \includegraphics[width=\linewidth]{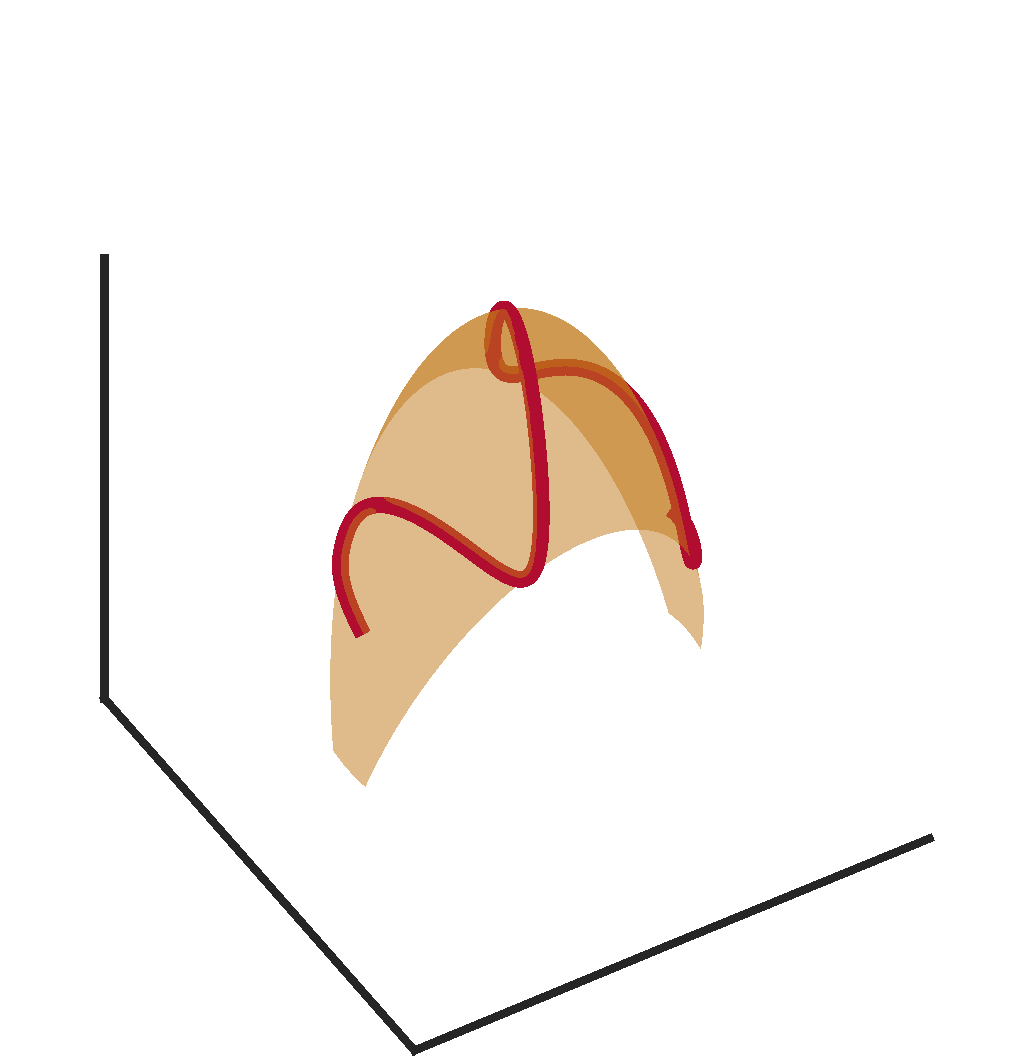}
        \label{fig:b-k-visualization:step-3}
    \end{subfigure} 
    \caption{A visualization of the embedding gap. In all three figures we plot $\hlred{f(K)}$ and $\hlorange{g_i(W)}$ for Left: $i = 1$, Center: $i = 2$ and Right: $i = 3$. Visually, we see that $\hlorange{g_i(W)}$ approaches $\hlred{f(W)}$ as {$i$} increases, and we compute $B_{K,W}(\hlred{f},\hlorange{g_1}) > B_{K,W}(\hlred{f},\hlorange{g_2}) > B_{K,W}(\hlred{f},\hlorange{g_3}) = 0$.}
    \label{fig:b-k-visualization}
\end{figure}

\subsection{Manifold Embedding Property}
\label{sec:main-results:manifold-embedding}

We now introduce a central concept, the manifold embedding property (MEP). A family of networks has the MEP if it can, as measured by the embedding gap, nearly embed a large class of manifolds of certain dimension and regularity. The MEP is a property of a family of functions $\cE \subset \emb({W},\Rea^m)$ where $W \subset \Rea^o$. In this manuscript, $\cE$ will always be formed by taking $\cE \coloneqq \cT \circ \cR$, where $\cR$ and $\cT$ are the expansive layers and bijective flow networks described in sections \ref{sec:expans} and \ref{sec:expres} respectively.


We note here that $\cE$ having the MEP is closely related to the question of whether or not a given $n$-dimensional manifold $\cM= f(K)$ for $f \in \emb(K,\R^m)$, $K\subset \Rea^n$, can be approximated by an $E \in \cE$. This choice of first applying (possibly non-universal) expansive layers, and then universal layers puts some topological restrictions on the expressivity, which we discuss in great detail in Section \ref{sec:top-obstructions}.

In anticipation of these topological difficulties, when we refer to the MEP, we consider it with respect to a class of functions $\cF \subset \emb(\Rea^n,\Rea^m)$. The MEP can be interpreted as a density statement, saying that our networks $\cE$ are dense in some set $\cF \subset \emb(\Rea^n,\Rea^m)$ in the topology induced by the `$B_{K,W}$ distance.' 
Two examples of $\cF$ that we are particularly interested in are the following. When $\cF = \emb(\Rea^n,\Rea^m)$, and also when each $f \in \cF$ can be written as $f = D \circ L$ where $L:\R^{m\times n}$ is a linear map of rank $n$ and $D:\R^m\to \R^m$ is a $C^k$ diffeomorphism with $k\ge 1$.

\begin{definition}[Manifold Embedding Property]
    \label{def:manif-embedding-problem}
    Let $\cE \subset \emb(\Rea^o,\Rea^m)$ and $\cF\subset \emb(\Rea^n,\Rea^m)$ be two families of functions. We say that $\cE$ has the $m,n,o$ Manifold Embedding Property (MEP) w.r.t. $\cF$ if for every compact {non-empty} set $K \subset \Rea^n$, $f \in \cF$, and $\epsilon > 0$, there is an $E \in \cE$ and a compact set $W \subset \Rea^o$
    such that the restriction of $f$ to $K$ and the restriction of $E$ to $W$ satisfies
    \begin{align}
        B_{K,W}(f,E) < \epsilon.
    \end{align}
    When it is clear from the context, we abbreviate the $m,n,o$ MEP w.r.t. $\cF$ simply by the $m,n,o$ MEP, or simply the MEP.
\end{definition}

We also present the following two lemmas which relate to the algebra of the MEP. 

\begin{lemma}
    \label{lem:comp-mep}
    Let $\cE^{p,o}_1 \subset \emb(\Rea^p,\Rea^o)$ have the $o, n, p$ {MEP} w.r.t. $\cF^{n,o}_1 \subset \emb(\Rea^n, \Rea^o)$, and likewise let $\cE^{o,m}_2 \subset \emb(\Rea^o,\Rea^m)$ have the $m, o, o$ MEP w.r.t. $\cF^{o,m} _2\subset \emb(\Rea^o, \Rea^m)$. If each $E^{o,m}_2 \in \cE^{o,m}_2$ is locally Lipschitz, then $\cE^{o,m}_2 \circ \cE^{p,o}_1$ has the $m, n, p$ MEP w.r.t. $\cF^{o,m}\circ \cF^{n,o}$.
\end{lemma}

The proof of Lemma \ref{lem:comp-mep} is in Appendix \ref{sec:lem:comp-mep}. 

We note that when the elements of $\cE^{o,m}_2$ are differentiable, local Lipschitzness is automatic, and need not be assumed, see e.g. \citep[Ex. 10.2.6]{tao2009analysis}. We also record the following lemma, proved in \ref{sec:lem:comp-mep-requires-mep}, which is a weak-converse of Lemma \ref{lem:comp-mep}. It states that if $\cE^{o,m}_2 \circ \cE^{p,o}_1$ has the $m,n,p$ MEP, then $\cE^{o,m}_2$ has the $m,n,o$ MEP.

\begin{lemma}
    \label{lem:comp-mep-requires-mep}
    Let $\cE^{p,o}_1 \subset \emb(\Rea^p,\Rea^o)$ and $\cE^{o,m}_2 \subset \emb(\Rea^o,\Rea^m)$ be such that $\cE^{o,m}_2 \circ \cE^{p,o}_1$ has the $m,n,p$ MEP with respect to family $\cF\subset  \emb(\R^n,\R^m)$. Then $\cE^{o,m}_2$ has the $m,n,o$ MEP with respect to family $\cF$.
\end{lemma}

\begin{definition}[Uniform Universal Approximator]
    \label{defn_uniform_universal_approximator}
    For a non-empty subset $\cF^{n,m} \subset C(\Rea^n,\Rea^m)$,  a family $\cE^{n,m} \subset  C(\Rea^n,\Rea^m)$ is said to be a uniform universal approximator of $\cF^{n,m}$ if for every $f \in \cF^{n,m}$, every non-empty compact $K \subset \Rea^n$, and each $\epsilon > 0$, there is an $E \in \cE^{n,m}$ satisfying:
    \begin{align}
        \sup_{x \in K} \norm{f(x) - E(x)}_{2} < \epsilon.
    \end{align}
\end{definition}

If $\cE \subset \emb(\R^o,\R^m)$ is a uniform universal approximator of $\cF^{o,m} = C^0(\Rea^o,\Rea^m)$ on compact sets, then it has the $m,n,o$ MEP w.r.t  $C^0(\Rea^n,\Rea^m)$ for any $n \leq o$, see  Lemma \ref{lem:mep-implied-from-univ}. As an example, when $m\ge 2o+1$ injective ReLU networks $E:\Rea^o\to \Rea^m$ (i.e., mappings of the form (R4)) are uniform universal approximator of $ C^0(\Rea^o,\Rea^m)$ on compact sets, see e.g. \citep{puthawala2020globally} and \citep{yarotsky2017error,yarotsky2018optimal}. 
Thus, networks that are uniform universal approximators automatically possess the MEP. Generalizations of this are considered in Lemma \ref{lem:mep-implied-from-univ}. 

With the definition of the MEP and uniform universal approximator established, we now discuss in detail the nature of the topological obstructions to approximating all one-chart manifolds. 

\subsection{Topological Obstructions to Manifold Learning with Neural Networks}
\label{sec:top-obstructions}

We show that using non-universal expansive layers and flow layers imposes some topological restrictions on what can be approximated. Let $n=2$, $m=3$, and $K = S^1\subset \R^2$ be the circle, and let 
\begin{align*}
    \cE = \set{T \circ R \in C(\R^2,\R^3) \colon R\in \R^{3\times 2}, T \in \hom(\R^3,\R^3)}.
\end{align*}
That is, $\cE$ is the set of maps that can be written as compositions of linear maps {from $\R^2$ to $\R^3$} and homeomorphisms {on all of $\R^3$}. Let $f \in \emb(K,\Rea^3)$ be an embedding that maps $K$ to a trefoil knot $\cM=f(S^1)$, see Fig. \ref{fig:knot}. Such a function $f$ can not be written as a restriction of an $E\in \cE$ to $S^1$. In Sec. \ref{sec:mapping-s-1-to-trefoil} we prove this fact and build a related example where a measure, $\mu \in \cP(\Rea^2)$, supported on an annulus is pushed forward to a measure supported on a knotted ribbon in $\Rea^3$ by an embedding $g \colon \Rea^2 \to \Rea^3$. For this measure, there are no $E \in \cE$ such that $\pushf{g}{\mu} = \pushf{E}{\mu}$. {We note that the counterexample is still valid if $\cE$ is replaced with $\hat \cE = \cT \circ \cD$ where $\cT = \hom(\R^3,\R^3)$ and $\cD = \hom(\R^3,\R^3) \circ \R^{3 \times 2}.$  
See \ref{sec:linear-homeomorphism-composition} for a proof. The point here is not that $R$ is linear, but rather that it embeds all of $\R^2$ into $\R^3$, rather than only $S^1$ into $\R^3$.
}

\begin{figure}
    \centering
    \begin{subfigure}{.45\linewidth}
        \centering
        \includegraphics[width=\linewidth]{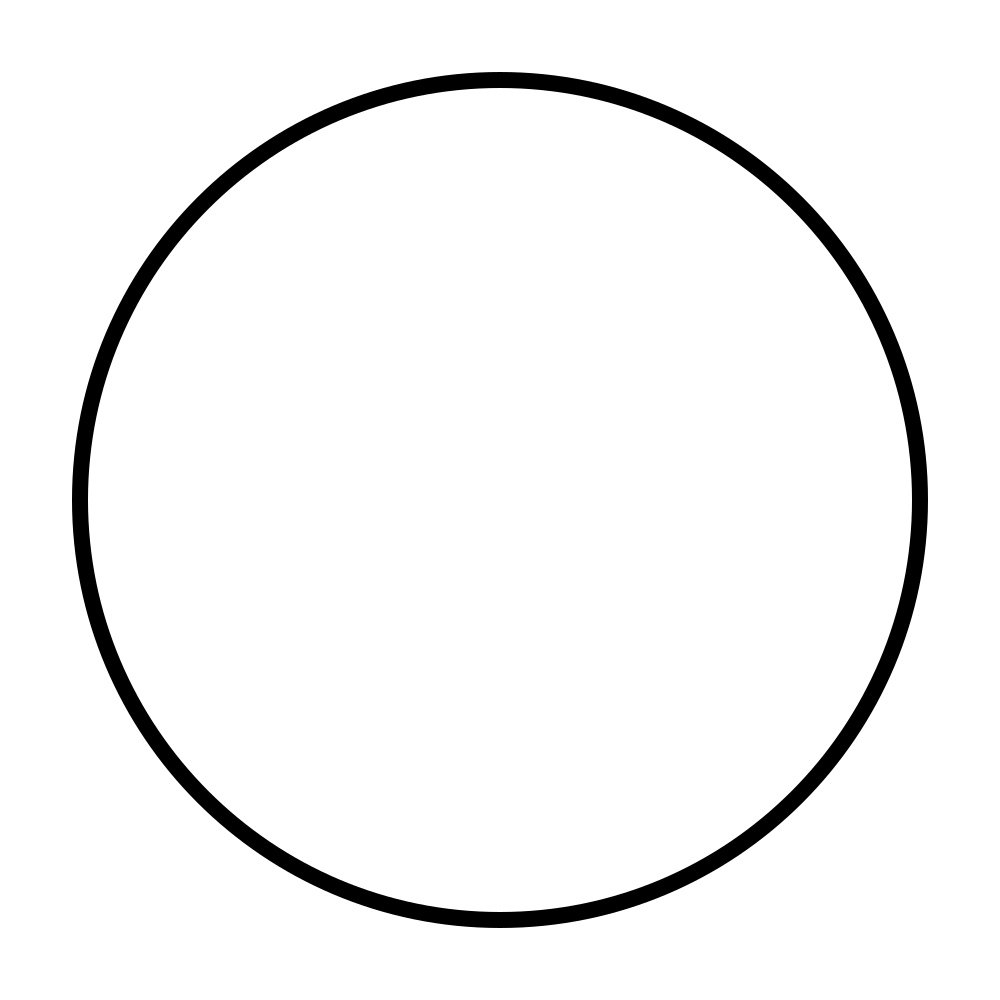}
    \end{subfigure}
    \begin{subfigure}{.45\linewidth}
        \centering
        \includegraphics[width=\linewidth]{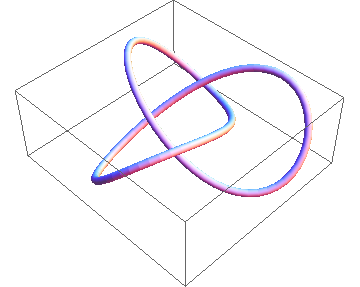}
    \end{subfigure}
    \caption{An illustration of the case when $n=2$, $m=3$, and $K = S^1$ is the circle. Here $f:S^1\to \Rea^3$ is an embedding such that the curve $\cM=f(S^1)$ is a trefoil knot. Due to knot theoretical reasons, there are no map $E=T\circ R:\R^2\to \R^3$ such that $E(S^1)=\cM$, where  $R:\R^2\to \R^3$ is a full rank linear map and $T:\R^3\to \R^3$ is a homeomorphism. This shows that a combination of linear maps and coupling flow maps can not represent all embedded manifolds. For this reason, we define the class ${\mathcal I}(\R^n,\R^m)$ of extendable embeddings $f$ in Definition \ref{def:extendable-embeddings}. A similar 2-dimensional example can be obtained to a knotted ribbon, see Sec.\ \ref{sec:mapping-s-1-to-trefoil}.
    }
    \label{fig:knot}
\end{figure}

With this difficulty in mind, we define the MEP property with respect to a certain subclass of manifolds $\{f(K):\ f\in \cF\}$.
%
%
Additionally, when considering flow networks which are universal approximators of $C^2$ diffeomorphisms, we restrict the class of manifolds to be approximated even further. This is necessary because manifolds that are homeomorphic are not necessarily diffeomorphic\footnote{A classic example are the exotic spheres. These are topological structures that are homeomorphic, but not diffeomorphic, to the sphere \citep{milnor1956manifolds}.}. Moreover, it is known that $C^2$-smooth diffeomorphisms can not approximate general homeomorphisms in the $C^0$ topology, see \citep{muller2014uniform} for a precise statement. All $C^1$-smooth diffeomorphisms $f:\R^m\to \R^m$, however, can be approximated in the strong topology of $C^1$ by $C^2$-smooth diffeomorphism $\tilde f:\R^m\to \R^m$, $\ell\ge k$, see \citep[Ch.\ 2, Theorem 2.7]{hirsch2012differential}. Because of this, we have to pay attention to the smoothness of the maps in the subset $\cF\subset \emb(K,\R^m).$

\begin{definition}[Extendable Embeddings]
    \label{def:extendable-embeddings}
    We define the set of Extendable Embeddings as
    \begin{align*}
        {\mathcal I}(\R^n,\R^m) &\coloneqq \cD \circ \cL \\
        \cD &= \hbox{Diff}^1(\R^m,\R^m)\\
        \cL &= \set{L \in \R^{m\times n}\colon \hbox{rank}(L) = n},
    \end{align*}
    where $\hbox{Diff}^k(\R^m,\R^m)$ is the set of $C^k$-smooth diffeomorphisms from $\R^m$ to itself. Note that ${\mathcal I}(\R^n,\R^m)\subset\emb(\R^n,\R^m)$.
\end{definition}

The word extendable in the name extendable embeddings refers to the fact that the family $\cD$ in Definition \ref{def:extendable-embeddings} is a proper subset of $\emb(L(K),\R^m)$ for some compact $K \subset \R^n$ and linear $L \in \R^{m\times n}$. Mappings in the set $\cD$ are embeddings $D:L(K)\to \R^m$ that extend to diffeomorphisms from all of $\R^m$ to itself. Said differently, a $D \in \cD$ is a map in $\emb^1(L(K),\R^m)$ that can be extended to a map $\tilde D \in \hbox{Diff}^1(\R^m,\R^m) $ such that $\restr{\tilde D}{L(K)} = D$. This distinction is important, as there are maps in $\emb^1(L(K),\R^m)$ that can not be extended to diffeomorphisms on all of $\R^m$, as can be seen from the counterexample developed at the beginning of this section.

We also present here a theorem that states that when $m$ is more than three times larger than $n$, any differentiable embedding from compact $K\subset \R^n$ to $\R^m$ is necessarily extendable.

\begin{theorem} \label{thm:emb-and-extens-coincide}
    When $m\ge 3n+1$ and $k\ge 1$,  for any $C^k$
    embedding ${f}\in {\emb}^{k}(\R^n,\R^m)$ and compact set $K\subset \R^n$, there is a map  ${E}\in {\mathcal I}^k(\R^n,\R^m)$ (that is, $E$ is in the closure of the set of flow type neural 
    networks) such that $E(K) = f(K)$. Moreover,
    \begin{align}\label{I-formula B}
        {\mathcal I}^k(K,\R^m)={\emb}^{k}(K,\R^m)
    \end{align}
\end{theorem}

The proof of Theorem \ref{thm:emb-and-extens-coincide} in Appendix \ref{sec:thm:emb-and-extens-coincide}. We also remark here that the proof of the above theorem relies on the so called `clean trick' from differential topology. This trick is related to fact that in $\R^4$, all knots can be reduced to the simple knot continuously.

\subsection{Universality}
\label{sec:main-results:universality-results}

We now combine the notions of universality and extendable embeddings to produce a result stating that many commonly used networks of the form studied in Section \ref{sec:desc-of-arch} have the MEP.

\begin{lemma}
    \label{lem:mep-implied-from-univ}
    \begin{enumerate}
        \item[(i)] If $\cR\subset \emb(\Rea^n,\Rea^m)$ is a uniform universal approximator of $C(\Rea^n,\Rea^m)$ and $I \in \cT$ where $I$ is the identity map, then $\cE \coloneqq \cT\circ \cR$ has the MEP w.r.t. $\emb(\Rea^n,\Rea^m)$.
        \item[(ii)] If $\cR$ is such that there is an injective $R \in \cR$ and open set $U\subset \Rea^o$ such that $\restr{R}{\overline U}$ is linear, and $\cT$ is a $\sup$ universal approximator in the space of $\hbox{Diff}^2(\R^m,\R^m)${, in the sense of \citep{teshima2020coupling},} of the $C^2$-smooth diffeomorphisms, then $\cE\coloneqq \cT\circ\cR$ has the MEP w.r.t. $\cI(\Rea^n,\Rea^m)$.
    \end{enumerate}
\end{lemma}

For uniform universal approximators that satisfy the assumptions of (i), see e.g. \citep{puthawala2020globally}. The proof of Lemma \ref{lem:mep-implied-from-univ} is in Appendix \ref{sec:lem:mep-implied-from-univ}.
It has the following implications for the architectures studied in Section \ref{sec:desc-of-arch}.

\begin{example}
    \label{examp:lots-of-archs-have-mep}
    Let $\cE \coloneqq \cT\circ\cR$ and (T1), (T2), (R1), \ldots, (R4) be as described in Section \ref{sec:desc-of-arch}. Then
    \begin{enumerate}
        \item[(i)] If $\cT$ is either (T1) or (T2) and $\cR$ is (R4), then $\cE$ has the $m,n,o$ MEP w.r.t. $\emb(\Rea^n,\Rea^m)$.
        \item[(ii)] If $\cT$ is (T2) with sigmoidal activations \citep{huang2018neural}, then if $\cR$ is any of (R1), ..., (R4), then $\cE$ has the $m,n,o$ MEP w.r.t. $\cI(\Rea^n,\Rea^m)$.
    \end{enumerate}
\end{example}

The proof of Example \ref{examp:lots-of-archs-have-mep} is in Appendix \ref{sec:examp:lots-of-archs-have-mep}.

We now present our universal approximation result for networks given in Eqn. \ref{eqn:network-def} and a decoupling property. Below, we say that a measure $\mu$ in $\mathbb R^n$ is absolutely continuous if it is absolutely continuous w.r.t. the Lebesgue measure.

\begin{theorem}
    \label{thm:univ-qual}
    Let $n_0 = n, n_L = m$ $K \subset \R^n$ be compact, $\mu \in \cP(K)$ be an absolutely continuous measure. Further let, for each $\ell = 1,\dots,L$, $\cE_\ell^{n_{\ell-1}, n_{\ell}} \coloneqq \cT^{n_\ell}_\ell \circ \cR^{n_{\ell-1}, n_\ell}_\ell$ where $\cR^{n_{\ell - 1}, n_\ell}_\ell$ is a family of injective expansive elements that contains a linear map, and $\cT^{n_\ell}_\ell$ is a family of bijective family networks. Finally let $\cT^{n}_{0}$ be distributionally universal, i.e. for any absolutely continuous $\mu\in\cP(\Rea^n)$ and $\nu\in\cP(\Rea^n)$, there is a $\set{T_i}_{i=1,2,\dots}$ such that $\pushf{T_i}{\mu} \to \nu$ in distribution. Let one of the following two cases hold:
    
    (i) $f \in \cF_{L}^{n_{L-1}, m}\circ\dots\circ\cF_{1}^{n, n_{1}}$ and $\cE_\ell^{n_{\ell-1}, n_{\ell}}$ have the the $n_{\ell}$, $n_{\ell-1}$,$ n_{\ell-1}$ MEP for $\ell = 1,\dots,L$ with respect to $\cF_\ell^{n_{\ell-1}, n_{\ell}}$.
    
    (ii) $f\in \emb^1(\R^n, \R^m)$ be a $C^1$-smooth embedding, for $\ell = 1,\dots,L$ $n_\ell\ge 3n_{\ell - 1}+1$ and the families $\cT^{n_\ell}_\ell$ are dense in $\hbox{Diff}^2(\R^{n_\ell})$.
    
    Then, there is a sequence of {$\set{E_{i}}_{i = 1,2,\dots} \subset \cE_L^{n_{L-1}, m} \circ \dots \circ \cE_1^{n_1, n} \circ \cT^{n}_0$} such that
    \begin{align}
        \label{eqn:qual-univ:qual}
        \lim_{i \to \infty}
        \wasstwo{\pushf{f}{\mu}}{\pushf{E_i}{\mu}} = 0.
    \end{align}

\end{theorem}

The proof of Theorem \ref{thm:univ-qual} is in Appendix \ref{sec:proofs:thm:univ-qual}.
{The results of Theorems \ref{thm:emb-and-extens-coincide} and \ref{thm:univ-qual} have a simple interpretation, omitting some technical details. Densities on `nice' manifolds embedded in high-dimensional spaces can always be approximated by neural networks of the form \eqref{eqn:network-def}. Here, `nice' manifolds are smooth and homeomorphic to $\R^n$. This proves that networks like Eqn. \ref{eqn:network-def} are `up to task' of solving generation problems.
}

As discussed in the above and in Figure \ref{fig:knot}, there are topological obstructions to obtaining the results of Theorem \ref{thm:univ-qual} with a general embedding $f:\R^n\to \R^m$. When $n=2$, $m=3$, $L=1$, and $\mu$ is the uniform measure on an annulus $K\subset \R^2$ target measure $F_\#\mu$ is the uniform measure on a knotted ribbon $\cM=f(K)\subset \R^3$. There are no injective linear maps $R:\R^2\to \R^3$  and diffeomorphisms $T:\R^3\to \R^3$ such that $E=T\circ R$ would satisfy $\cM=E(K)$ and $E_\#\mu=F_\#\mu$.   
%

We note that our networks are designed expressly to approximate manifolds, and hence injectivity is key. This separates our results from, e.g. \citep[Theorem 3.1]{lee2017ability} or \citep[Theorem 2.1]{lu2020universal}, where universality results of $\relu$ networks are also obtained. 

The previous theorem states that the entire network is universal if it can be broken into pieces that have the MEP. The following lemma, proved in Appendix \ref{sec:proofs:mep-is-nesc-for-universality}, shows that if $\cE^{n,m} = \cH^{o,m}\circ \cG^{n,o}$, then $\cH^{o,m}$ must have the $m,n,o$ MEP if $\cE^{n,m}$ is universal. 

\begin{lemma}
    \label{lem:mep-is-nesc-for-universality}
    Suppose that $\cE^{n,m} = \cH^{o,m} \circ \cG^{n,o}$ where $\cE^{n,m} \subset \emb(\Rea^n,\Rea^m)$, $\cH^{o,m} \subset \emb(\Rea^o,\Rea^m)$, and $\cG^{n,o} \subset \emb(\Rea^n,\Rea^o)$. If $\cH^{o,m}$ does not have the $m,n,o$ MEP w.r.t. $\cF$, then there exists a $f \in \cF$, compact $K \subset \Rea^n$ and $\epsilon > 0$ such that for all $E \in \cE^{n,m}$, and $r \in \emb(f(K),E(W))$
    \begin{align}
        \norm{I - r}_{L^{\infty}(K)} \geq \epsilon.
    \end{align}
\end{lemma}

\begin{figure}[t]
    \centering
    \begin{subfigure}{.30\linewidth}
        \centering
        \includegraphics[width=\linewidth]{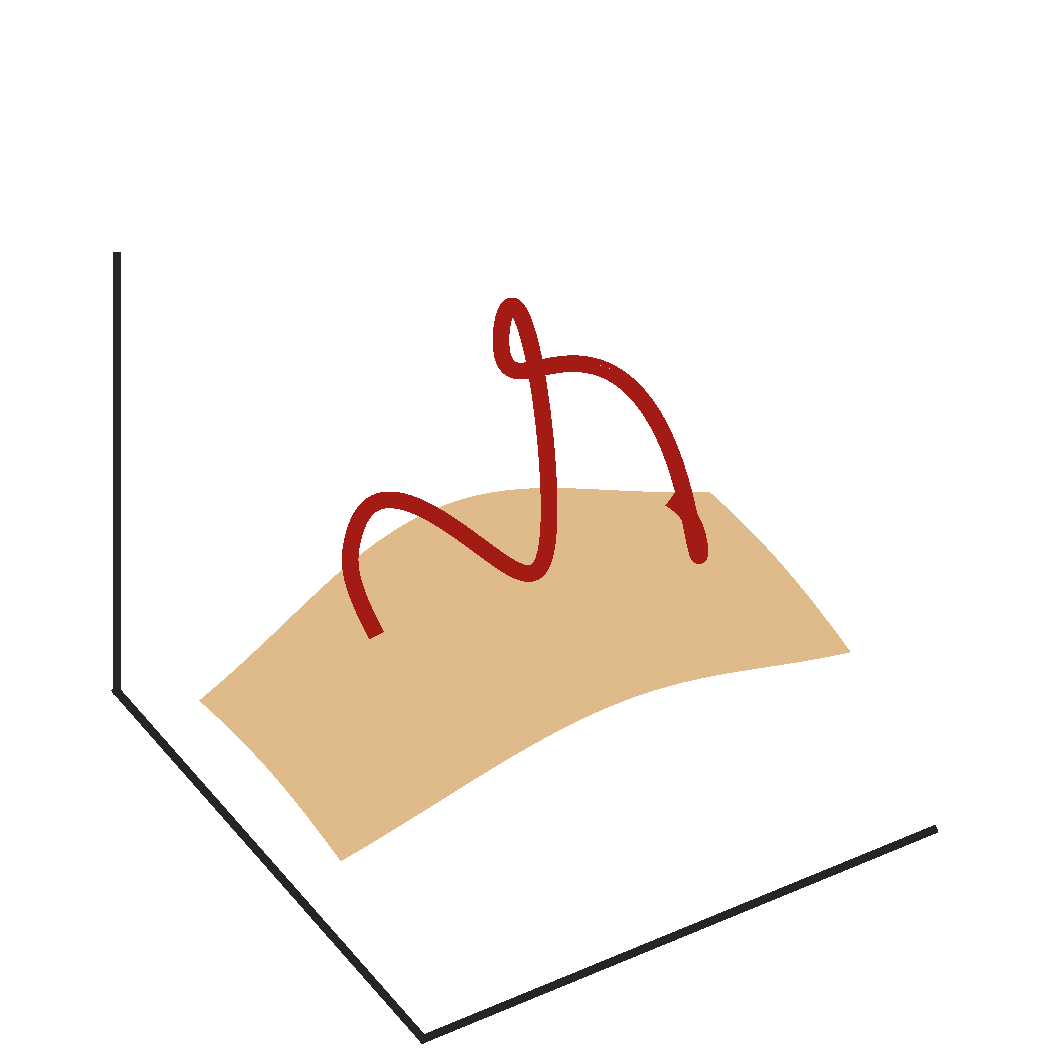}
        \subcaption{$f$ and $E_1$}
    \end{subfigure}
    \begin{subfigure}{.30\linewidth}
        \centering
        \includegraphics[width=\linewidth]{Approxing_first_Manifold__Step_3.png}
        \subcaption{$f$ and $E_2$}
    \end{subfigure}
    \begin{subfigure}{.30\linewidth}
        \centering
        \includegraphics[width=\linewidth]{Approxing_first_Manifold__Step_4.png}
        \subcaption{$f$ and $E_3$}
    \end{subfigure}
    \\
    \begin{subfigure}{.30\linewidth}
        \centering
        \includegraphics[width=\linewidth]{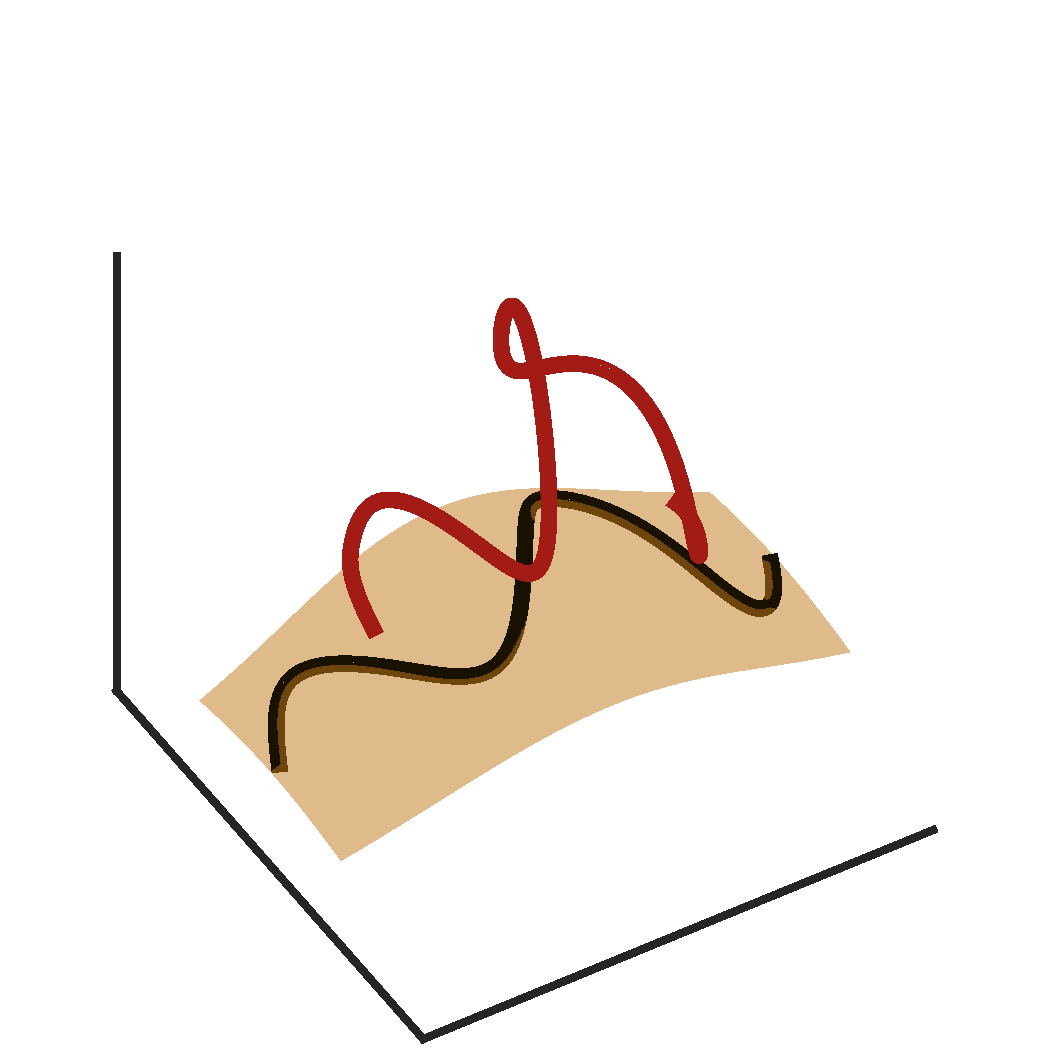}
        \subcaption{$E_1\circ E'_1$}
        \label{fig:parallel-convergence:step-1}
    \end{subfigure}
    \begin{subfigure}{.30\linewidth}
        \centering
        \includegraphics[width=\linewidth]{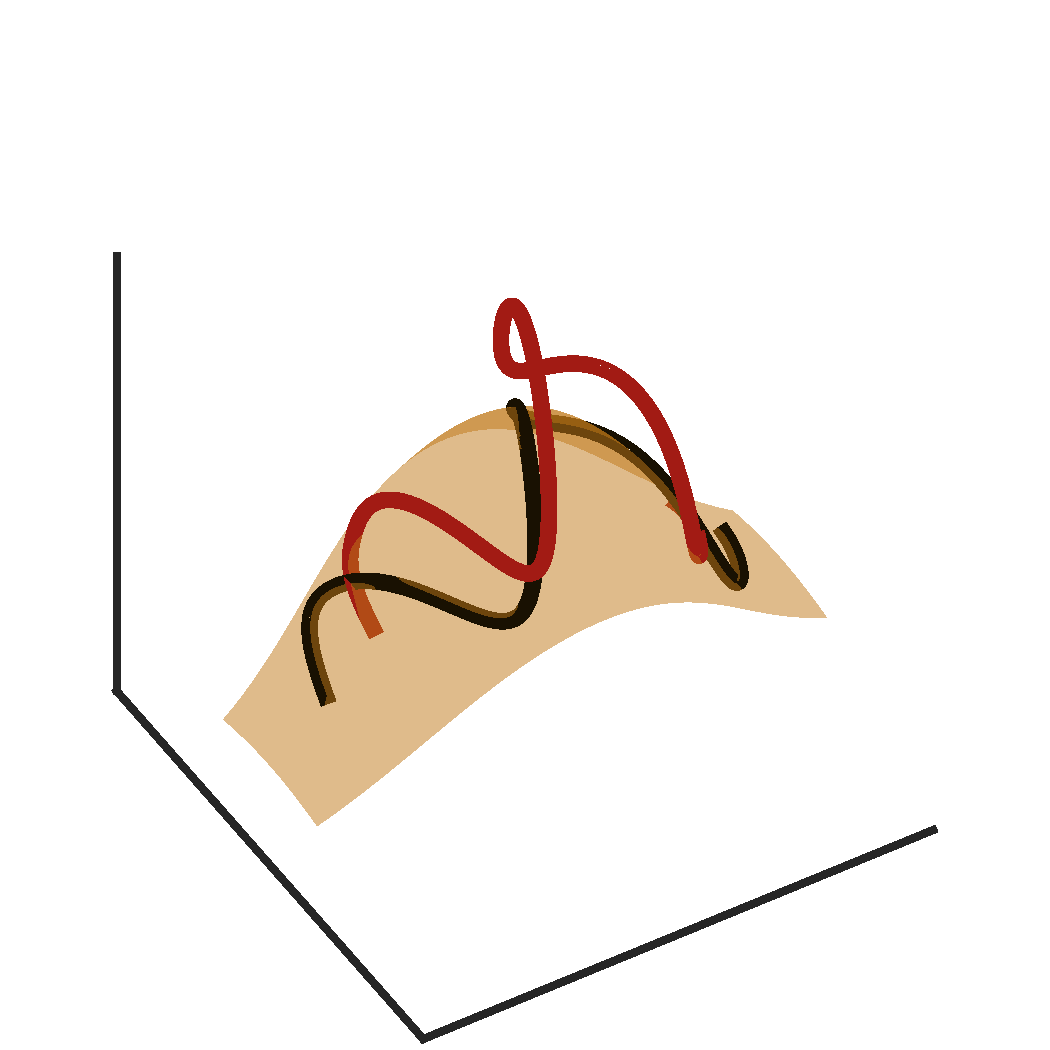}
        \subcaption{$E_2\circ E'_2$}
        \label{fig:parallel-convergence:step-2}
    \end{subfigure}
    \begin{subfigure}{.30\linewidth}
        \centering
        \includegraphics[width=\linewidth]{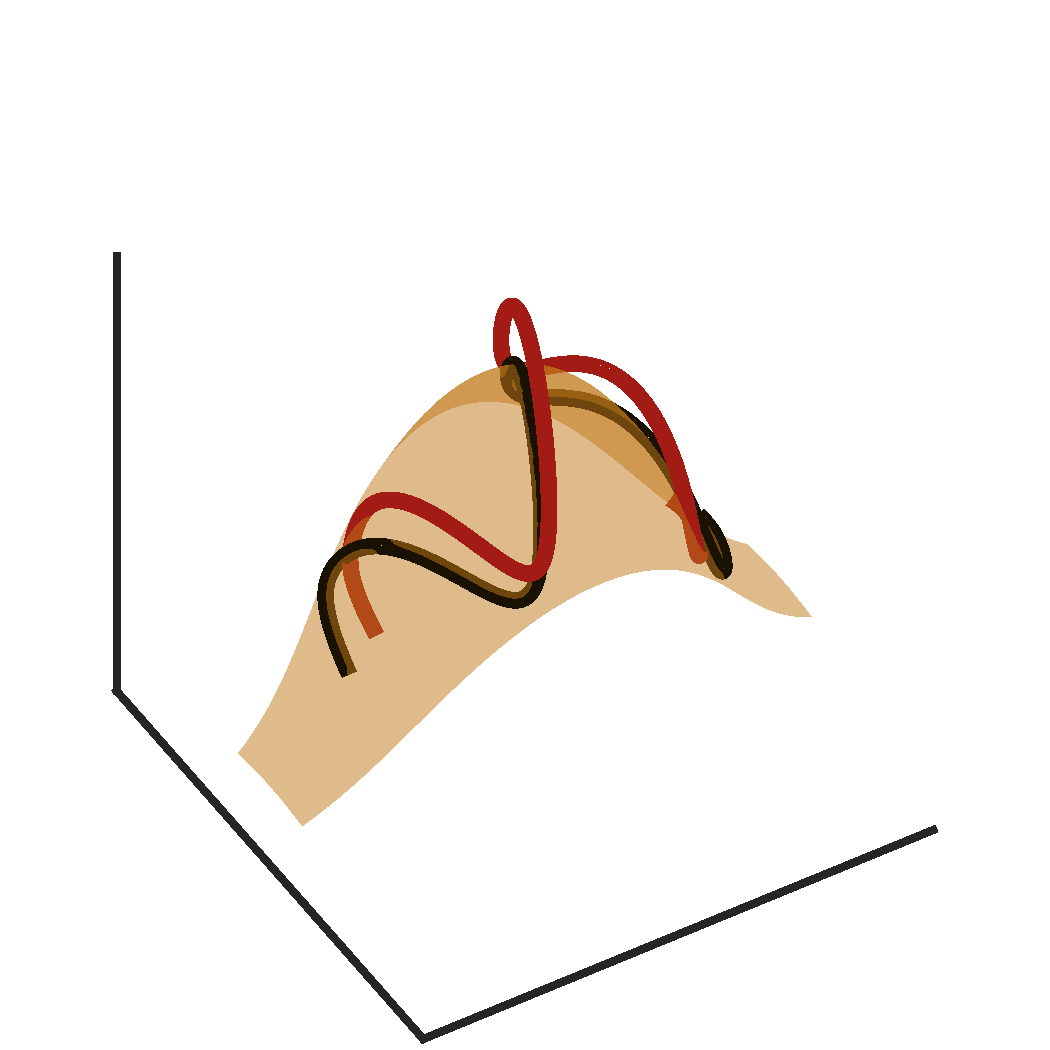}
        \subcaption{$E_3\circ E'_3$}
        \label{fig:parallel-convergence:step-3}
    \end{subfigure} 
    \caption{A visualization of the construction described in Corollary \ref{cor:decoupling} applied to a toy example when $m = 3, o = 2$ and $n = 1$. In all figures, \hlred{$f(K)$} is the \hlred{red} curve, \hlorange{$E_i(W)$} are the \hlorange{orange} surfaces, \hlblack{$E_i\circ E_i'(W')$} are the \hlblack{black} curves, $T_i$ and $\mu$ are not pictured. (a) - (c) The \hlorange{orange} surfaces approach the \hlred{red} curves. This means that the sequence of \hlorange{$E_1$}, \hlorange{$E_2$} and \hlorange{$E_3$} send $B_{K,W}(\hlred{f},\hlorange{E_i})$ to zero as $i$ increases. (d) - (f) The \hlblack{black} curves, a subset of the \hlorange{orange} surfaces, approach the \hlred{red} curves. This means that given \hlorange{$E_1$}, \hlorange{$E_2$} and \hlorange{$E_3$} we can always find another sequence $E_1'$, $E_2'$ and $E_3'$ that sends $B_{K,W}(\hlred{f},E_i\circ E_i')$ to zero as $i$ increases too. This as a consequence, sends $\wasstwo{\pushf{\hlred{f}}{\mu}}{\pushf{E_i\circ E'_i\circ T_i}{\mu}}$ to zero as $i$ increases too for some choice of $T_1$, $T_2$ and $T_3$.}
    \label{fig:parallel_convergence}
\end{figure}

{Lemma \ref{lem:mep-is-nesc-for-universality} has a simple takeaway: If a bijective neural network is universal, then the last layer, last two layers, etc., must have the MEP. In other words, a network is only as universal as its last layer. Earlier layers, on the other hand, need not satisfy the MEP. `Strong' layers close to the output can compensate for `weak' layers closer to the input, but not the other way around.}

{There is a gap between the negation of Theorem \ref{thm:univ-qual} and Lemma \ref{lem:mep-is-nesc-for-universality}. That is, it is possible for a family of functions $\cE$ to satisfy Lemma \ref{lem:mep-is-nesc-for-universality} but nevertheless satisfy the conclusion of Theorem \ref{thm:univ-qual}; these functions approximate measures without matching manifolds. Theorem \ref{thm:univ-qual} considers approximating measures, whereas Lemma \ref{lem:mep-is-nesc-for-universality} refers to matching manifolds exactly. As discussed in Section \ref{sec:top-obstructions}, there are no extendable embeddings that map $S^1$ to the trefoil knot in $\R^3$. Nevertheless, it is possible to construct a sequence of functions $\paren{E_i}_{i = 1,\dots}$ so that $\wasstwo{\nu}{\pushf{E_i}{\mu}} = 0$ where $\mu$ and $\nu$ are the uniform distributions on $S^1$ and trefoil knot respectively. Such a construction is given in \ref{sec:approximate-vs-exact-matching}. 

Although there are sequences of functions that approximate measure without matching manifolds, these sequences are never uniformly Lipschitz. This is proven in \ref{sec:approximate-vs-exact-matching}. Under an idealization of training, we may consider a network undergoing training as successively better and better approximators of a target mapping. If the target mapping does not match the topology, then training necessarily leads to gradient blowup.
}

The proof of Theorem \ref{thm:univ-qual} also implies the following result which, loosely speaking, says that optimality of later layers can be determined without requiring optimality of earlier layers, while still having a network that is end-to-end optimal. The conditions and result of this is visualized on a toy example in Figure \ref{fig:parallel_convergence}. 

\begin{corollary}
    \label{cor:decoupling}
    Let $\cF^{n,o} \subset \emb(\Rea^n,\Rea^o)$, $\cF^{o,m}\subset \emb(\Rea^o,\Rea^m)$, and let $\cE^{o,m} \subset \emb(\Rea^o,\Rea^m)$ have the $m,n,o$ MEP w.r.t. $\cF^{o, m}\circ \cF^{n,o}$. Then for every $f \in \cF^{o,m}\circ \cF^{n,o}$ and compact sets $K \subset \Rea^n$ and $W \subset \R^o$ there is a sequence $\set{E_i}_{i = 1,2,\dots} \subset \cE^{o,m}$ such that
    \begin{align}
        \label{eqn:decoupling:later-layers-min:1}
        \lim_{i\to\infty}B_{K,W}(f,E_i)
        = 0.
    \end{align}
    Further,if there is a compact $W' \subset \R^n$ and $\cE^{n,o} \subset \emb(W',\Rea^o)$ has the $o,n,n$ MEP w.r.t. $\cF^{n,o}$, and a $\cT^n$ is a universal approximator for distributions, then for any absolutely continuous $\mu\in\cP(K)$ where $K\subset \Rea^n$ is compact,  there is a sequence $\set{E_{i}'}_{i = 1,2,\dots} \subset \cE^{n,o}$ and $\set{T_{i}}_{i = 1,2,\dots,} \subset \cT^n$ so that
    \begin{align}
        \label{eqn:decoupling:later-layers-min:2}
        \lim_{i\to\infty}\wasstwo{\pushf{f}{\mu}}{\pushf{E_i\circ E'_i \circ T_i}{\mu}} = 0.
    \end{align}
\end{corollary}

The proof of Corollary \ref{cor:decoupling} is in Appendix \ref{sec:cor:decoupling}. Approximation results for neural networks are typically given in terms of the network end-to-end. Corollary \ref{cor:decoupling} shows that the layers of approximating networks can in fact be built one at a time. This is related to an observation made in \citep[Section B]{brehmer2020flows} about training strategies, where the authors remark that they `expect faster and more robust training of a network' of the form in Eqn. \ref{eqn:network-def} when $L = 1$, that is $\cF = \cT^{m}_1\circ \cR^{n,m}_1 \circ \cT^n_0$. Corollary \ref{cor:decoupling} shows that there exists a minimizing sequence in $\cT^{m}_1$ that need only minimize Eqn. \ref{eqn:decoupling:later-layers-min:1}; the $\cT^n_0$ layers can be minimized after. 
We can further combine Lemma \ref{lem:mep-is-nesc-for-universality} and Cor. \ref{cor:decoupling} to prove that not only can the network from \citep{brehmer2020flows} be trained layerwise, but that \textit{any} universal network can \textit{necessarily} be trained layerwise, provided that it can be written as a composition of two smaller layers.

\subsection{Layer-wise Inversion and Recovery of Weights}
\label{sec:additional-properties}

In this subsection, we describe how our network can be augmented with more useful properties if the architecture satisfies a few more assumptions without affecting universal approximation. We focus on a new layerwise projection result, with a further discussion of black-box recovery of our network's weights in Appendix \ref{sec:additional-prop-details:black-box-recovery}.

\begin{figure}
    \centering
    \includegraphics[width=.5\linewidth]{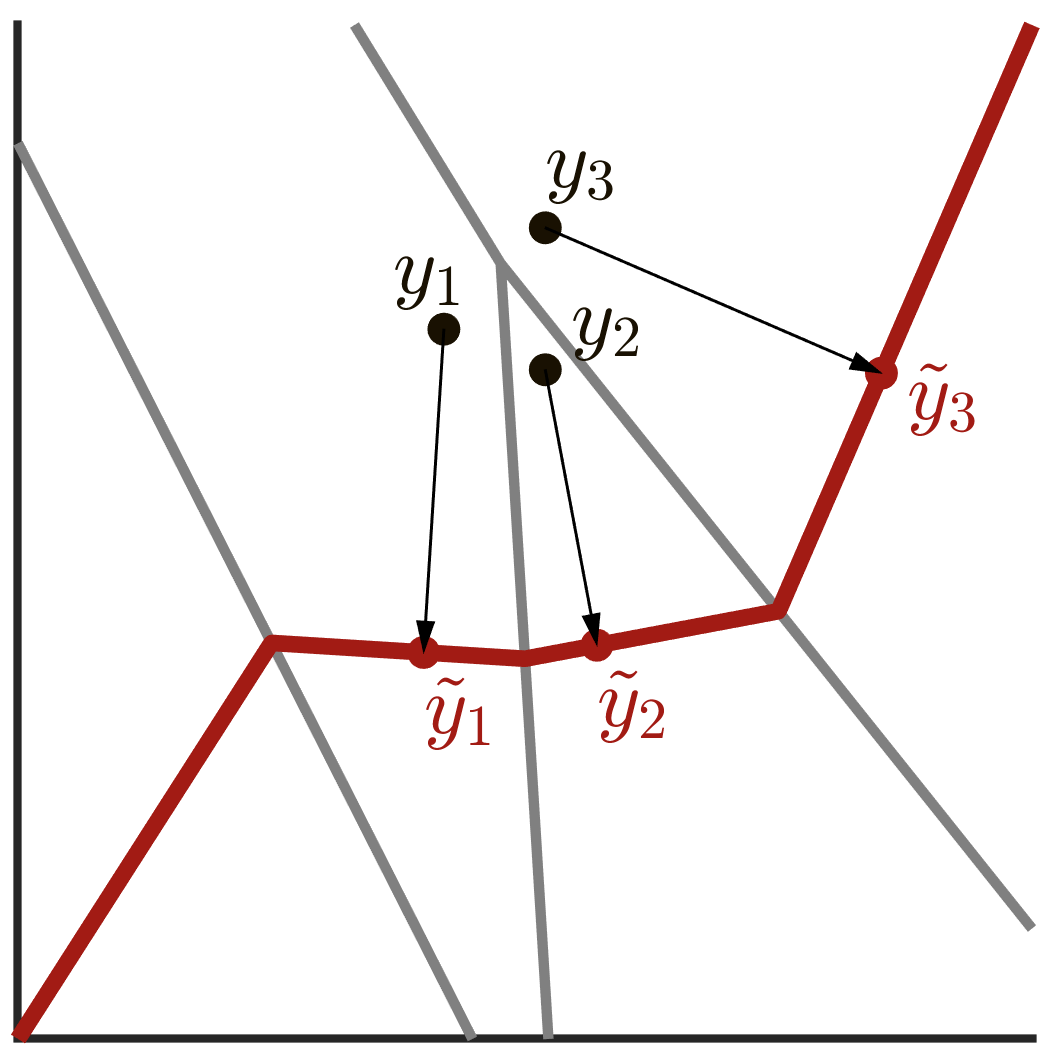}
    \caption{A schematic showing that, for a toy problem, the least-squares projection to a piecewise affine range can be discontinuous. Left: A partitioning of $\Rea^2$ into classes with \hlgray{gray} boundaries. Two points $y, y'$ are in the same class if they are both closest to the same affine piece of \hlred{$R(\Rea)$}, the range of $\hlred{R}$. The three points $y_1, y_2$ and $y_3$ are each projected to the closest three points on \hlred{$R(\Rea)$} yielding \hlred{$\tilde y_1$}, \hlred{$\tilde y_2$} and \hlred{$\tilde y_3$}. Note that the projection operation is continuous within each section, but discontinuous across \hlgray{gray} boundaries between section.
    }
    \label{fig:least-squares-proj}
\end{figure}

Given a point $y \in \Rea^{m}$ that does \ul{not} lie in the range of the network, projecting $y$ onto the range of the network is a practical problem without an obvious answer.
The crux of the problem is inverting the injective (but non-invertible) $\cR$ layers when $\cR$ contains only full-rank matrices as in (R1) or (R2) then we can compute a least-squares solution. If, however, $\cR$ contains layers which are only piecewise linear, as in (R3), then the problem of computing a least squares solution is more difficult, see Fig. \ref{fig:least-squares-proj}. Nevertheless, we find that if $\cR$ is (R3) we can still compute a least-squares solution.



\begin{assumption}
    Let $\cR$ be given by one of (R1) or (R2), or else (R3) when $m = 2n$.
\end{assumption}

If $\cR$ only contains linear operators, then the least-squares problem can be computed by solving the normal equations (see \citep[Section 5.3]{golub1996matrix}.) This includes cases (R1) or (R2). For (R3) we have the following result when $D = I^{n\times n}$ and $M \in \R^{0\times n}$.

\begin{definition}
    \label{def:one-layer-proj}
    Let $W = \bmat{B^t&-DB^t}^t \in \Rea^{2n\times n}$  and $y\in \R^{2n}$ be given, and let $R(x) = \relu(Wx)$. Then define $c(y) \in \Rea^{2n}, \Delta_{y}\in \Rea^{n\times n}, M_{y}\in \Rea^{n\times 2n}$ where
    \begin{align}
        \label{eqn:projection:c-y}
        c(y) &\coloneqq \max\paren{\bmat{I^{n\times n}&-I^{n\times n} \\ -I^{n\times n} & I^{n\times n}}y, 0}\\
        \bracketed{\Delta_{y}}_{i,j} &\coloneqq \begin{cases}
                0 &\text{if } i \neq j\\
                0 &\text{if } \bracketed{c(y)}_{i + n} = 0\\
                1 &\text{if } \bracketed{c(y)}_{i + n} > 0
        \end{cases}\\
        M_y &\coloneqq \bmat{(I^{n\times n} -\Delta_{y}) & \Delta_{y}}
    \end{align}
    where the max in Eqn. \ref{eqn:projection:c-y} is taken element-wise.
\end{definition}

\begin{theorem}
    \label{thm:proj-lemma}
    Let $y \in \Rea^{2n}$. If for $i = 1,\dots,n$,  $\bracketed{y}_{i} \neq \bracketed{y}_{i + n}$ then
    \begin{align}
        \label{eqn:thm:proj-lemma}
        R^{\dagger}(y) \coloneqq \paren{M_yW}^{-1}M_y y = \argmin_{x \in \Rea^n}\norm{y - R(x)}_2.
    \end{align}
    Further, if there is a $i \in \set{1,\dots,n}$ such that $\bracketed{y}_{i} = \bracketed{y}_{i + n}$, then there are multiple minimizers of $\norm{y - R(x)}_2$, one of which is $R^{\dagger}(y)$.
\end{theorem}

The proof of Theorem \ref{thm:proj-lemma} is given in Appendix \ref{sec:additional-prop-details:layer-wise-projection}.

\begin{remark}
    We note that Theorem \ref{thm:proj-lemma} is different from many of the existing work on inverting expansive layers, e.g. \citep{aberdam2020and,bora2017compressed,lei2019inverting}, our result gives a direct inversion algorithm that is provably the least-squares minimizer. Further, if each expansive layer is any combination of (R1), (R2), or (R3) then the entire network can be inverted end-to-end by using either the above result or solving the normal equations directly.
\end{remark}

\section{Conclusion}
\label{sec:conclusion}

Bijective flow networks are a powerful tool for learning push-forward mappings in a space of fixed dimension. Increasingly, these flow networks have been used in combination with networks that increase dimension in order to produce networks which are purportedly universal. 

In this work, we have studied the theory underpinning these flow and expansive networks by introducing two new notions, the embedding gap and the manifold embedding property. We show that these notions are both necessary and sufficient for proving universality, but require important topological and geometrical considerations which are, heretofore, under-explored in the literature. We also find that optimality of the studied networks can be established `in reverse,' by minimizing the embedding gap, which we expect opens the door to convergence of layer-wise training schemes. Without compromising universality, we can also use specific expansive layers with a new layerwise projection result. 

\section{Acknowledgements}
We would like to thank Anastasis Kratsios for his editorial input and mathematical discussions that helped us refine and trim our presentation; Pekka Pankka for his suggestion of the `clean trick,' which was crucial to the development of the proof of Lemma \ref{thm:emb-and-extens-coincide}; and Reviewer for supplying \ref{fig:unknot-to-trefoil-knot} and suggesting the addition of \ref{sec:approximate-vs-exact-matching}.

I.D. was supported by the European Research Council Starting Grant 852821---SWING. M.L. was  supported by Academy of Finland, grants 284715, 312110. M.V.dH. gratefully acknowledges support from the Department of Energy under grant DE-SC0020345, the Simons Foundation under the MATH + X program, and the corporate members of the Geo-Mathematical Imaging Group at Rice University.


\bibliography{imports/references}
\bibliographystyle{icml2022}
\newpage
\appendix
\onecolumn

\section{Summary of Notation}
\label{sec:summary-of-notation}

Throughout the paper we make heavy use of the following notation.

\begin{enumerate}
    \item Unless otherwise stated, $X$ and $Y$ always refer to subsets of Euclidean space, and $K$ and $W$ always refer to compact subsets of Euclidean space.
    \item $f \in C(X,Y)$ means that $f \colon X \to Y$ is continuous.
    \item For families of functions $\cF$ and $\cG$ where each $\cF\ni f \colon X \to Y$ and $\cG \ni g \colon Y \to Z$, then we define $\cG\circ\cF = \set{g\circ f \colon X \to Z \colon f \in \cF, \ g \in \cG}$.
    \item $f \in \emb(X,Y)$ means that $f\in C(X,Y)$ is continuous and injective on the range of $f$, i.e. an embedding, and furthermore that $f^{-1}\colon f(X) \to X$ is continuous.
    \item $\mu \in \cP(X)$ means that $\mu$ is a probability measure over $X$.
    \item $\wasstwo{\mu}{\nu}$ for $\mu,\nu \in \cP(X)$ refers to the Wasserstein-2 distance, always with $\ell_2$ ground metric.
    \item $\norm{\cdot}_{L^p(X)}$ refers to the $L^p$ norm of functions, from $X$ to $\Rea$.
    \item For vector-valued $f\colon X \to Y$,  $\norm{f}_{L^\infty(X)} = \esssup_{x \in \cX}\norm{f}_{2}$. Note that $Y$ is always finite dimensional, and so all discrete $1 \leq q\leq \infty$ norms are equivalent.
    \item $\lip(g)$ refers to the Lipschitz constant of $f$.
    \item For $x \in \Rea^n$, $\bracketed{x}_i \in \Rea$ is the $i$'th component of $x$. Similarly, for matrix $A \in \Rea^{m\times n}$, $[A]_{ij}$ refers to the $j$'th element in the $i$'th column.
\end{enumerate}

\section{Detailed Comparison to Prior work}
\label{sec:comparison-to-others}

\subsection{Connection to Brehmer \& Cranmer (2020)}
\label{sec:comparison-to-others:brehmer}

In \citep{brehmer2020flows}, the authors introduce manifold-learning flows as an invertible method for learning probability density supported on a low-dimensional manifold. Their model can be written as

\begin{align}
    \label{eqn:brehmer-kramer-architecture}
    \cF = \cT_1^{m}\circ \cR^{n,m}\circ \cT_0^n
\end{align}
where $\cT^m_{1} \subset C(\Rea^m,\Rea^m)$, $\cT^m_{0}\subset C(\Rea^n,\Rea^n)$, and $\cR = \set{\bmat{I^{n\times n}\\ \bszero^{(m-n)\times n}}}$ is a zero-padding (R1). They invert $f \in \cF$ in two different ways. For manifold-learning flows ($\cM$-flows) they restrict $\cT_{1}^m$ to be an invertible flow, and for manifold-learning flows with separate encoder ($\cM_e$-flows) they place no such restrictions on $\cT^m_{1}$ and instead train a separate neural network $e$ to invert elements of $\cT^m_{1}$.

Our results apply out-of-the-box to the architectures used in Experiment A of \citep{brehmer2020flows}. The architecture described in Eqn. \ref{eqn:brehmer-kramer-architecture} is of the form of Eqn. \ref{eqn:network-def} where $L = 1$. Further, although they are not studied here, our analysis can also be applied to quadratic flows.

The network used in \citep[Experiment 4.A]{brehmer2020flows} uses coupling networks, (T1), where $\cT^m_{1}$ and $\cT^n_{0}$ are both 5 layers deep. 
For \citep[Experiments 4.B and 4.C]{brehmer2020flows} the authors choose expressive elements $\cT$ that are rational quadratic flows \citep{durkan2019neural} for both $\cT^m_{1}$ and $\cT^n_{0}$. In Experiment 4.B they let $T_{1}$ and $T_{0}$ again be 5 layers deep, 
and in 4.C they again let $T_{1}$ by 20 layers deep and $T_{0}$ 15 layers. 
For the final experiment, 4.D, the choose more complicated expressive elements that combine Glow \citep{kingma2018glow} and Real NVP \citep{dinh2016density} architectures. These elements include the actnorm, $1\times1$ convolutions and rational-quadratic coupling transformations along with a multi-scale transformation. 

The authors mention universality of their network without our proof, but our universality results in  Theorem \ref{thm:univ-qual} apply to their networks from Experiment A wholesale. Further in their work the authors describe how training can be split into a manifold phase and density phase, wherein the manifold phase $\cT^m_{1}$ is trained to learn the manifold, and in the density phase $\cT^m_{1}$ if fixed and $\cT^n_{0}$ is trained to learn the density thereupon. This statement is made formal and proven by our Cor. \ref{cor:decoupling}.

\subsection{Connection to Kothari et al. (2021)}
\label{sec:comparison-to-others:kothari}

In \citep{kothari2021trumpets}, the authors introduce the `Trumpet' architecture, for its architecture, which has many alternating flow networks \& expansive layers with many flow-networks in the low-dimensional early stages of the network, which gives the architecture a shape similar to the titular instrument.

The architecture studied in \citep{kothari2021trumpets} is precisely of the form of Eqn. \ref{eqn:network-def}, where the bijective flow networks are revnets \citep{gomez2017reversible,jacobsen2018revnet} architecture, and the expansive elements are $1\times 1$ convolutions, as in (R2). To out knowledge, there are no results that show that the revnets used are universal approximators, but if they revnets are substituted with either (T1) or (T2), then the, we could apply Theorem \ref{thm:univ-qual} to the resulting architecture.


\section{Proofs}
\subsection{Main Results}
\label{sec:proofs:preliminaries}

\subsubsection{Embedding Gap}
\label{sec:proofs:preliminaries:helper-lemma}
To aid all of our subsequent proofs, we first present the following lemma which present inequalities and identities for the embedding gap.

\begin{lemma}
    \label{lem:b-k-helper}
    For all of the following results, $f \in \emb(K,\Rea^m)$ and $g \in \emb(W,\Rea^m)$ and $n \leq o \leq m$.
    \begin{enumerate}
        \item 
            \begin{align}
                B_{K,W}(f,g) &\geq \sup_{x_n \in K}\inf_{x_o \in W} \norm{g(x_o) - f(x_n)}_2.
            \end{align}
        \item Let $X,Y \subset W$, let $g$ be Lipschitz on $W$, and $r \in \emb(X,Y)$.
        Then, there is a $r' \in \emb(g(X),g(Y))$ such that $g\circ r = r'\circ g$ and $\norm{I - r'}_{L^{\infty}(g(X))} \leq \norm{I - r}_{L^{\infty}(X)}\lip(g)$.
        \item 
            \begin{align}
                \norm{I - r}_{L^{\infty}(K)} = \norm{I - r^{-1}}_{L^{\infty}(r(K))}
            \end{align}

        \item
            Let $K\subset \R^n$, $X\subset \R^p$ and $W\subset \R^o$ be compact sets. Also, let $f\in \emb(K,W)$ and $h\in \emb(X,W)$, and let  $g\in \emb(W,\R^m)$ be
            a Lipschitz map. Then
            \begin{align}
                B_{K,X}(g\circ f,g\circ h) &\leq \lip(g)  B_{K,X}(f,h) .
          \end{align}
        \item $B_{K,W}(f,g) \leq \sup_{x \in K} \norm{g\circ h(x) - f(x)}_2$
        where $h \in \emb(K,\Rea^o)$  is a map satisfying $h(K)\subset W$.
        
        \item For any $X$ that is the closure of an open set , if $h \in \emb(X,W)$ then 
        \begin{align}
            B_{K,W}(f,g) \leq B_{K,X}(f,g\circ h)
        \end{align}
        
        \item For any $r \in \emb(f(K),\Rea^m)$, \begin{align}
            B_{K,W}(f,g) \leq \norm{I - r}_{L^{\infty}(f(K))} + B_{K,W}(r\circ f, g).
        \end{align}
        \item For any $r\in \emb(f(K),g(W))$ and $h \in \emb(X,W)$ where $X \subset \Rea^p$ is the closure of a set $U$ which is open in the subspace topology of some vector space of dimension $p$, where $n \leq p \leq o$ we have that 
        \begin{align}
            B_{K,X}(f,g\circ h) \leq \norm{I - r}_{L^{\infty}(f(K))} + \lip(g) B_{K,X}(g^{-1}\circ r \circ f, h)
        \end{align}
        where $\lip(g)$ denotes the Lipschitz constant of $g$.
        \item 
            For any $\mu_n \in \cP(K)$ there is a $\mu_o \in \cP(W)$ such that
            \begin{align}
                \label{eqn:wass-lower-bounded-by-b}
                \wasstwo{\pushf{f}{\mu_n}}{\pushf{g}{\mu_o}} \leq  B_{K,W}(f,g)
            \end{align}
   \end{enumerate}
\end{lemma}

\begin{proof}
    \begin{enumerate}
        \item Let $r \in C(f(K),g(W))$, then 
        \begin{align*}
            \norm{I - r}_{L^{\infty}(f(K))} &= \sup_{x_n \in K}\norm{(I - r)f(x_n)}_{2} = \sup_{x_n \in K}\norm{f(x_n) - r\circ f(x_n)}_2 \\
            &= \sup_{x_n \in K}\norm{f(x_n) - g(x_o)}_2 \text{ where } x_o = g^{-1}\circ r\circ f(x_n)\\
            &\geq \sup_{x_n \in K}\inf_{x_o \in W}\norm{f(x_n) - g(x_o)}_2. 
        \end{align*}
        \item $g$ is injective on $X$, hence we can define $r'$ such that $r' = g\circ r \circ g^{-1}: g(X)\to g(r(X))\subset g(Y)$ such that $r'\in \emb(g(X),g(Y))$, and thus $\forall x \in X$, 
        \begin{align}
            \norm{\paren{I - r'}\circ g(x)}_2 = \norm{g(x) - g\circ r (x)}_2 \leq \lip(g) \norm{I - r}_{L^{\infty}(X)}
        \end{align}
        where we have used $\norm{r(x) - x}_2 \leq \norm{I - r}_{L^{\infty}(X)}$.
        
        \item For every $x \in r(K)$, we have a $y \in K$ such that $x = r(y)$, thus $\forall x \in r (K)$,
        \begin{align}
            \norm{\paren{I - r^{-1}}(x)}_{2} = \norm{\paren{r - I}(y)}_{2}.
        \end{align}
        But $r$ is clearly surjective onto it's range, hence taking the supremum over all $x \in X$ yields 
        \begin{align}
            \norm{I - r^{-1}}_{L^{\infty}(r(K))} = \norm{I - r}_{L^{\infty}(K)}
        \end{align}
        
        \item  As $g\in \emb(W,\R^m)$, the map $g:W\to g(W)$ is a homeomorphism
              and there is $g^{-1}\in \emb(g(W), W)$. 
              For a map $r\in \emb(g\circ f(K),g\circ h(X))$, we see
              that $\hat r=g^{-1}\circ r\circ g\in  \emb( f(K),h(X))$. Also, the opposite is valid as 
              if $\hat r\in  \emb( f(K),h(X))$ then $r=g\circ \hat r\circ g^{-1}\in \emb(g\circ f(K),g\circ h(X))$.
              Thus
             \begin{align*}
                B_{K,X}(g\circ f,g\circ h) &=\inf_{r\in \emb(g\circ f(K),g\circ h(X))}    \| I - r\|_{L^\infty(g\circ f(K))}\\
                 &=\inf_{r=g\circ \hat r\circ g^{-1} \in \emb(g\circ f(K),g\circ h(X))}    \| I- g\circ \hat r\circ g^{-1} \|_{L^\infty(g\circ f(K))}\\
                 &=\inf_{\hat r\in  \emb( f(K),h(X))}    \|  g\circ (I-\hat r)\circ g^{-1} \|_{L^\infty(g\circ f(K))}
\\
                 &\le \lip(g) \inf_{\hat r\in  \emb( f(K),h(X))}    \| (I-\hat r)\circ g^{-1} \|_{L^\infty(g\circ f(K))}
                 \\
                 &\le \lip(g) \inf_{\hat r\in  \emb( f(K),h(X))}    \| I-\hat r \|_{L^\infty(f(K))}
\\
                 &\le \lip(g)  \, B_{K,X}(f,h)
              \end{align*}
        \item If we let $r \coloneqq g\circ h \circ f^{-1}$, then $r \in \emb(f(K),g(W))$, and 
        \begin{align}
            B_{K,W}(f,g) &\leq \norm{\norm{\paren{I - r}\circ f(x)}_2}_{L^\infty(K)} \\
            &= \norm{\norm{f(x) -  g\circ h \circ f^{-1}\circ f(x)}_2}_{L^\infty(K)} \leq \sup_{x \in K} \norm{f(x) - g \circ h(x)}_2.
        \end{align}
        \item Given that $g\circ h(X) \subset g(W)$, we have that $\emb(f(K),g\circ h(X)) \subset \emb(f(K),g(W))$, thus the infimum in Eqn. \ref{eqn:b-k-def} is taken over a smaller set, thus $B_{K,W}(f,g) \leq B_{K,X}(f,g\circ h)$.
        \item Note that for any $r' \in \emb(r\circ f(K),g(W))$, $r' \circ r \in \emb(f(K),g(W))$, and so we have
        \begin{align}
            B_{K,W}(f,g) \leq \norm{I - r' \circ r}_{L^{\infty}(f(K))} &\leq \norm{I - r}_{L^{\infty}(f(K))} + \norm{r - r'\circ r}_{L^{\infty}(f(K))} \\
            &= \norm{I - r}_{L^{\infty}(f(K))} + \norm{I - r'}_{L^{\infty}(r\circ f(K))}
        \end{align}
        where we have used that $r$ is injective for the final equality. This holds for all possible $r'$, hence we have the result.
        \item 
        Recall that $f\in \emb(K,W)$, $g\in \emb(W,\R^m)$,
$h \in \emb(X,W)$ and  $r\in \emb(f(K),g(W))$. Then 
$g^{-1}\in \emb(g(W), W)$. As $ r\circ  f(K)\subset g(W)$, we see that
      \begin{align*}
          r\circ f=g\circ g^{-1}     \circ r \circ f.   
           \end{align*}
              Thus Lemma \ref{lem:b-k-helper} points 4 and 8 yield that
                    \begin{align*}
            B_{K,X}(f,g\circ h) &\leq \norm{I - r}_{L^{\infty}(f(K))} + B_{K,X}( r \circ f, g\circ h)\\
            &\leq \norm{I - r}_{L^{\infty}(f(K))} + B_{K,X}( g\circ g^{-1}     \circ r \circ f, g\circ h)\\
     &\leq \norm{I - r}_{L^{\infty}(f(K))} + \lip(g)\, B_{K,X}(g^{-1}     \circ r \circ f, h),
        \end{align*}
              which proves the claim.

        \item Let $r_\epsilon\in \emb(f(K),g(W))$ be such that $\norm{I - r_\epsilon}_{L^{\infty}\paren{\range(f)}} \leq B_{K,W}(f,g) + \epsilon$, then for every $x \in K$, there exists $y \in W$ such that $g(y) = r \circ f(x)$. From injectivity of $g$, we have that $y = g^{-1}\circ r_\epsilon \circ f(x)$. Note that $g^{-1}\circ r_\epsilon \circ f \in \emb(K,W)$, hence $K' \coloneqq g^{-1}\circ r_\epsilon \circ f(K) \subset W$ is compact. Define $\mu'_\epsilon \in \cP(K')$ where $\mu'_\epsilon \coloneqq \pushf{(g^{-1}\circ r_\epsilon \circ f)}{\mu}$. Clearly $\pushf{g}{\mu'_\epsilon} = \pushf{r_\epsilon \circ f}{\mu}$, and thus
            \begin{align}
                \wasstwo{\pushf{f}{\mu}}{\pushf{g}{\mu'_\epsilon}} = 
                \wasstwo{\pushf{ f}{\mu}}{\pushf{r_\epsilon\circ f}{\mu}}
            \end{align}
            and so
            \begin{align}
                \wasstwo{\pushf{f}{\mu}}{\pushf{g}{\mu'_\epsilon}} \leq \paren{\int_{K}\norm{I - r_\epsilon}_2^2 d\pushf{f}{\mu}}^{1/2} \leq B_{K,W}(f,g) + \epsilon.
            \end{align}
            As the set $W$ is compact, by Prokhoros's theorem, see \citep[Theorem 5.1]{Billingsley}, the set of probability measures $P(W)$ is a compact set in the topology of weak convergence. Thus there is a sequence $\epsilon_i\to 0$ such that the measures $\mu'_{\epsilon_i}$ converge weakly to a probability measure $\mu_o$. As $g:W\to K$ is a continuous function, the push-forward operation $\mu\to g_\# \mu$ is continuous $g_\# :P(W)\to P(K)$ and thus $ g_\#\mu'_{\epsilon_i} $ converge weakly to $ g_\#\mu_o$. Finally, as $g_\#\mu'_{\epsilon_i} $  are supported in a compact set $K$, their second moments converge to those of $g_\#\mu_o$ as $i\to \infty$. By \citep{Ambrosio1}, Theorem 2.7, see also Remark 28, the weak convergence and the convergence of the second moments imply the convergence in the  Wasserstein-2 metric.  Hence,  $ g_\#\mu'_{\epsilon_i} $ converge to $ g_\#\mu_o$ in Wasserstein-2 metric   and we see that
            \begin{align}
            \wasstwo{\pushf{f}{\mu}}{\pushf{g}{\mu_o}}  \leq B_{K,W}(f,g).
            \end{align}
    \end{enumerate}
\end{proof}




\subsection{Manifold Embedding Property}

\subsubsection{The Proof of Lemma \ref{lem:comp-mep}}
\label{sec:lem:comp-mep}

\begin{proof}[The proof of Lemma \ref{lem:comp-mep}]
    Let $f = F_2 \circ F_1$ where $F_2 \in \cF^{o,m}$ and $F_1 \in \cF^{n,o}$ and $\epsilon > 0$ be given, and let $E^{o,m}$. Clearly, $B_{K,W}(f,E) \leq B_{K,W}(F_2,E)$ and so by the $m,o,o$ MEP of $\cE^{o,m}$ with respect to $\cF^{o,m}$, we have the existence of an $r_m \in \emb(f(K),E^{o,m})$ such that $\norm{I - r}_{L^\infty(f(K))} < \epsilon$. $K_o \coloneqq \paren{E^{o,m}}^{-1}\circ r \circ f(K)$ is compact, hence $E^{o,m}$ is Lipschitz on $K_o$, so we can apply Lemma \ref{lem:b-k-helper} point 8, so 
    \begin{align}
        \label{eqn:comp-mep:1}
        B_{K,W}(f,E^{o,m} \circ E^{p,o}) \leq \norm{I - r}_{L^{\infty}(f(K))} + \lip(E^{o,m}) B_{K,W}(\paren{E^{o,m}}^{-1}\circ r \circ f, E^{p,o}).
    \end{align}
    But, because $f \in \cF^{o,m}\circ\cF^{n,o}$, we can choose a $E^{p,o}\in \cE_1^{p,o}$ so that $B_{K,W}(\paren{E^{o,m}}^{-1}\circ r \circ f, E^{p,o})\leq \frac{\epsilon}{2}\lip(E^{o,m})$ which, combined with Eqn. \ref{eqn:comp-mep:1}, proves the result. 
\end{proof}

\subsubsection{The Proof of Lemma \ref{lem:comp-mep-requires-mep}}
\label{sec:lem:comp-mep-requires-mep}

\begin{proof}[The proof of Lemma \ref{lem:comp-mep-requires-mep}]
Recall that $ \cF\subset \emb(\Rea^n,\Rea^m)$. Suppose that $\cE^{o,m}_2$ does not have the $m,n,o$ MEP with respect to $ \cF$, then there are some $\epsilon > 0$ and $f \in \cF$ 
so that
    \begin{align}
        \forall E^{o,m} \in \cE^{o,m}_2\
          \forall W_1\subset\subset \R^o
        , \quad B_{K,W_1}(f,E^{o,m}_2) \geq \epsilon.
    \end{align}
    From Lemma \ref{lem:b-k-helper} point 6, we have that
    \begin{align}\label{set W1W}
        \epsilon \leq B_{K,W_1}(f,E^{o,m}_2) \leq B_{K,W}(f,E^{o,m}_2\circ E^{p,o}_1)
    \end{align}
    for all $E^{p,o}_1 \in \cE^{p,o}_1$
    and for all compact sets $W \subset \R^p$ that satisfy $ E^{p,o}_1(W_1)\subset W$.
    We observe that if  $W'\subset \R^p$ is a compact set
    such that $W'\subset W$, we have
    $$
    B_{K,W}(f,E^{o,m}_2\circ E^{p,o}_1)
    \leq B_{K,W'}(f,E^{o,m}_2\circ E^{p,o}_1)
    $$
    Thus, inequality \eqref{set W1W} holds for
    all $E^{p,o}_1 \in \cE^{p,o}_1$
    and for all compact sets $W \subset \R^p$.
    Summarising, we have seen that there are $f\in \cF$ and $\epsilon>0$ such that 
    for all $E^{p,o}_1 \in \cE^{p,o}_1$ and
    for all 
    compact sets $W\subset \R^p$ we have $\epsilon \leq B_{K,W}(f,E^{o,m}_2\circ E^{p,o}_1)$. Hence
$\cE^{o,m}_2$ does not have the $m,n,o$ MEP with respect to $ \cF$, and   
     we have obtained a contradiction, which proves the result.
\end{proof}

\subsection{Topological Obstructions to Manifold Learning with Neural Networks}

\subsubsection{$S^1$ can not be Mapped Extendably to the Trefoil Knot}\label{sec:mapping-s-1-to-trefoil}

We first show that there are no maps $E \coloneqq T \circ R$ where $R \colon \Rea^2 \to \Rea^3$ such that $T$ is a homeomorphism and $E(S^1)$ is a trefoil knot. We use the fact that the trivial knot $S^1$ and the trefoil knot $\cM=f(S^1)$ are not equivalent, that is, there are no homeomorphisms in $\Rea^3$ that map $S^1$ to $\cM$. Indeed, by \citep[Section 3.2]{Murasugi}, the trefoil knot $\cM$ and its mirror image are not equivalent, whereas the trivial knot $S^1$ and its mirror image are equivalent. Hence, $\cM$ and $R(S^1)$ are not equivalent knots in $\R^3$. Thus by \citep[Definition\ 1.3.1 and Theorem 1.3.1]{Murasugi}, we see that there is no orientation preserving homeomorphism $T:\R^3\to \R^3$ such that $T(\R^3\setminus R(S^1))=\R^3\setminus \cM.$ As the orientation of the map $T$ can be changed by composing $T$ with the reflection $J:\R^3\to \R^3$ across the plane $\range(R)$ that defines a homeomorphism $J:\R^3\setminus R(S^1)\to \R^3\setminus R(S^1)$, we see that there is no homeomorphism $T:\R^3\to \R^3$ such that $T(\R^3\setminus R(S^1))=\R^3\setminus \cM.$ 

This example shows that the composition $E=T\circ R$ of a linear map $R$ and a coupling flow $T$ cannot have the property that $E(S^1)=f(S^1)$ for this embedding $f$. Moreover, the complement $\R^3\setminus E(S^1)$ is never homeomorphic to  $\R^3\setminus f(S^1)$ for any such map $E$.

{ We now construct another example, similar to Figure \ref{fig:knot}, where an annulus that is mapped to a knotted ribbon
in $\R^3$. To do this, replace the circle $S^1$ by an annulus
$K=\{x\in \R^2:\ 1/2\le |x| \le 3/2\}$, that in the polar coordinates is $\{(r,\theta):\ 1/2\le r \le 3/2\}$ and define 
a map $F:K\to \R^3$ by defining in the polar coordinates  
$$
F(r,\theta)=f(\theta)+a (r-1)v(\theta)
$$
where $f:S^1\to \Sigma_1\subset \R^3$ is an smooth embedding of $S^1$ to a trefoil knot $\Sigma_1$ and 
$v(\theta)\in \R^3$ is a unit vector
normal to $\Sigma_1$ at the point $f(\theta)$ such that $v(\theta)$ is a smooth function of $\theta$,
and $a>0$  is a small number. In this case, $M_1=F(K)$ is a 2-dimensional submanifold of $\R^3$
with boundary, which can visualizes $M_1$ as a knotted ribbon. 

We now show that there are no maps $E = T\circ R$ such that $E(K) = F(K)$ where $T \colon \Rea^3 \to \Rea^3$ is an embedding, and $R \colon \Rea^2 \to \Rea^3$ injective and linear. The key insight is that if such a $T$ existed, then this implies that the trefoil knot is equivalent to $S^1$ in $\Rea^3$, which is known to be false. 

Let $U_\rho(A)$ denote the $\rho$-neighborhood
of the set $A$ in $\R^3$. It is easy to see that $\R^2\setminus ( \{0\}\times [-1,1])$ is homeomorphic to 
$\R^2\setminus \overline B_{\R^2}(0,1)$, which is further  homeomorphic to $\R^2\setminus\{0\}$.
Thus, using tubular coordinates near $\Sigma_1$ and a sufficiently small $\rho>0$, we see that $\R^3\setminus M_1$  is homeomorphic to $\R^3\setminus U_\rho(\Sigma_1)$, which is further  homeomorphic to $\R^3\setminus \Sigma_1$.
Also, when $R:\R^2\to \R^3$ is an injective linear map, we see that $M_2=R(K)$  is a un-knotted band in $\R^3$
and $\R^3\setminus M_2$   is  homeomorphic to $\R^3\setminus \Sigma_2$.
If $\R^3\setminus M_1$  and $\R^3\setminus M_2$ would be homeomorphic, 
then also $\R^3\setminus \Sigma_1$ and $\R^3\setminus \Sigma_2$ would be homeomorphic
that is not possible by knot theory, see \citep[Definition\ 1.3.1 and Theorem 1.3.1]{Murasugi}. This shows that there are no
injective linear maps $R:\R^2\to \R^3$ and  homeomorphisms  $\Phi:\R^3\to \R^3$
such that $(\Phi\circ R)(K)=M_1$.

Similar examples can be obtained in a higher dimensional case by
using a knotted torus \citep{Sequin}\footnote{On the knotted torus, see \url{http://gallery.bridgesmathart.org/exhibitions/2011-bridges-conference/sequin}.} and their Cartesian products.}

{
\subsubsection{Linear Homeomorphism Composition}
\label{sec:linear-homeomorphism-composition}
In this subsection we prove that the topological obstructions to universality presented in Section \ref{sec:top-obstructions} still apply when the expansive elements are allowed to be $\hom(\R^3,\R^3) \circ \R^{3\times 2}$. This fact follows from the observation that $\hom(\R^3,\R^3) \circ \hom(\R^3,\R^3) = \hom(\R^3,\R^3)$, which yields that $\hat \cE = \cE$. 

}

\subsubsection{The Proof of Theorem \ref{thm:emb-and-extens-coincide}}
\label{sec:thm:emb-and-extens-coincide}

{

Given an $f \in {\emb}^{k}(K,\R^m)$, for $k \geq 1$, we first show that for $m \geq 2n+1$ there is always a diffeomorphism $\Psi \colon \Rea^m \to \Rea^m$ so that $\Psi\circ f \colon \R^n \to \set{0}^{n}\times \Rea^{m-n}$. The existence of such a $\Psi$ borrows ideas from Whitney's embedding theorem \citep[Theorems 3.4 \& 3.5]{hirsch2012differential} and is constructed by iteratively constructing an injective projection.

Next if $m-n \geq 2n+1$, then we can apply \citep[Lemma 7.6]{madsen1997calculus}, a result analogous to the Tietze extension theorem, to show that $\Psi\colon \cM \to \set{0}^{n}\times \R^{m-n}$ can be extended to a diffeomorphism on the entire space, $h \colon \Rea^m \to \Rea^m$. Hence $f(x) = \Psi^{-1}\circ h \circ R(x)$ for diffeomorphism $\Psi^{-1}\circ h \colon \Rea^m \to \Rea^m$ and zero-padding operator $R \colon \Rea^n\to \Rea^m$, and thus $f \in \cI^k(K,\R^m)$.
This fact that for $m$ sufficiently large compared to $n$ such a diffeomorphism can always be extended is related to the fact that in 4-dimensions, all knots can be opened. This can be contrasted with the case in Figure \ref{fig:knot}.

We now present our proof.
}

\begin{proof}
 {   Let us next prove \eqref{I-formula B} when $m\ge 3n+1$. Let
      \begin{align}\label{f in I-formula}
            f\in {\emb}^k(\R^n,\R^m)
        \end{align}
    be a $C^k$ map  and $\cM=f(\R^n)$ be an embedded submanifold of $\Rea^m$.
    
  {We have that $m \geq 3n+1 > 2n+1$}. Let $S^{m-1}$ be the unit sphere of $\R^m$ and let $$S\R^m=\{(x,v)\in \R^m\times \R^m:\ \|v\|=1\}$$ be the sphere bundle of $\R^m$ that is a manifold of dimension $2m-1$. By the proof's of Whitney's embedding theorem, by Hirsch,  \citep[Chapter 1, Theorems 3.4 and 3.5]{hirsch2012differential}, there is a set {of `problem points'} $H_1\subset S^{m-1}$ of Hausdorff dimension $2n$ such that for all $w\in \R^m\setminus H_1$ the orthogonal projection $$P_w:\R^m\to \{w\}^\perp= \{y\in \R^m:\ y\perp w\}$$ has a restriction $P_w|_\cM$ on $\cM$ defines an injective map $$P_w|_\cM:\cM\to   \{w\}^\perp.$$ Moreover, let 
  {$T_x\cM$} be the tangent space of manifold $\cM$ at the point $x$ and let {us define another set of `problem points' as}
    $$
    H_2=\{ v\in S^{m-1}:\ \exists x\in \cM, v\in T_x{\cM}\}.
    $$
    For $w\in S^{m-1}\setminus H_2$  the map $$P_w|_\cM:\cM\to   \{w\}^\perp\subset \R^m$$ is an immersion, that is, it has an injective differential. The sphere tangent bundle $S\cM$  of $\cM$  has dimension $2n-1$, and the set $H_2$ has the Hausdorff dimension at most $2n-1$. Thus $H=H_1\cup H_2$ has Hausdorff dimension at most $2n<m-1$ and hence the set $S^{m-1}\setminus H$ is non-empty. For $w\in  S^{m-1}\setminus H$    the map $P_w|_\cM:\cM\to   \{w\}^\perp$ is a $C^k$ injective immersion and thus $$\tilde N=P_w(\cM)\subset \{w\}^\perp$$ is a $C^k$ submanifold.
    
    Let $Z:P_w(\cM)\to \cM$ be the $C^k$ function defined by $$Z(y)\in \cM,\quad P_w(Z(y))=y,$$ that is it is the inverse of $P_w|_\cM:\cM\to  P_w(\cM),$ where $   P_w(\cM{)}\subset \{w\}^\perp$. Let $g:\tilde N=P_w(\cM)\to \R$ be the function $$g(y)=(Z(y)-y)\cdot w,\quad y\in P_w(\cM).$$ 
    Then $\tilde N$ is a $n$-dimensional $C^k$ submanifold of $(m-1)$-dimensional Euclidean space $H= \{w\}^\perp$ and $g$ is a $C^k$ function defined on it.
By definition of a $C^k$ submanifold of  $H$, any point $x\in \tilde N$ has a neighborhood
$U\subset H$ with local $C^k$ coordinates $\psi:U\to \R^m$ such that $\psi(\tilde N\cap U)
=(\{0\}^{m-1-n}\times \R^n)\cap\psi(U)$. Using these coordinates, we see that $g$ can be extended 
to a  $C^k$ function in $U$. Using a suitable partition of unity, we see that  
     there is a $C^k$ map $G: \{w\}^\perp\to \R$ that a $C^k$ extension 
      of $g$ that is, $G|_{\tilde N}=g$. 
      
    Then the map
    $$\Phi_1:\R^m\to \R^m,\quad \Phi_1(x)=x-G(P_w(x))w$$  
    is a $C^k$ diffeomorphism of $\R^m$ that maps $\cM$ to $m-1$ dimensional space $ \{w\}^\perp$, that is $$\Phi_1(\cM)\subset \{w\}^\perp.$$ In the case when  $m\ge 3n+1$, we can repeat this construction $n$ times. This is possible as $m-n\ge 2n+1$. Then we obtain $C^k$ diffeomorphisms $\Phi_j:\R^m\to \R^m$, $j=1,\dots,n$  such that their composition $\Phi_n\circ\dots\circ \Phi_1:\R^m\to \R^m$ is a $C^k$-diffeomorphism such that which $$\cM'=\Phi_n\circ\dots\circ \Phi_1(\cM)\subset Y',$$ where $Y'\subset \R^m$ is a $m-n$ dimensional linear space. {By letting $\Psi = Q\circ \Phi_n\circ\dots\circ \Phi_1$ for rotation matrix $Q \in \Rea^{m\times m}$, we have} that $Y\coloneqq Q(Y') =\{0\}^n\times \R^{m-n}$. Also, let $X=\R^n\times \{0\}^{m-n},$ $A=Q(\cM')\subset X$ and $\phi: X\to \R^m$ be the map
    $$
    \phi(x,0)=\Psi(f(x))\in Y,
    $$
    where $f$ is the function given in \eqref{f in I-formula}
    and $B=\Psi(f(A))\subset Y$. Then $A$ is a $C^k$-submanifold $X$, $B$ is a $C^k$-submanifold $Y$ and $\phi:A\to B$ is a $C^k$-diffeomorphism. We observe that  $m-n\ge 2n+1$ and so we can apply \citep[Lemma 7.6]{madsen1997calculus} to extend $\phi$ to  a $C^k$-diffeomorphism  $$h:\R^{m}\to \R^m$$ such that $h|_A=\phi$. Note that {\citep[Lemma 7.6]{madsen1997calculus}} concerns {an} extension of a homeomorphism, but as the extension $h$ is given by an explicit formula {which is locally a finite sum of $C^k$ functions}, the same proof gives a $C^k$-diffeomorphic extension $h$ {to} a  diffeomorphism $\phi$. Indeed, let  $A'\subset \R^{n}$ and  $B'\subset \R^{m-n}$ be such sets that $A=A'\times \{0\}^{m-n}$,
    and $B=\{0\}^n\times B'$. Moreover,
    let
    $\tilde \phi:A'\to \R^{n-m}$ and 
    $\tilde \psi:B'\to \R^n$ be such $C^k$-smooth maps that
    $\phi(x,0)=(0,\tilde \phi(x))$ for $(x,0)\in A$
     and $\phi^{-1}(0,y)=(\tilde \psi(y))$ for $(0,y)\in B$.
    As $A'$ and $B'$ are $C^k$-submanifolds, the map $\tilde \phi$ has a $C^k$-smooth extension 
    $f_1:\R^{n}\to \R^{n-m}$ and the map $\tilde \psi$ has a $C^k$-smooth extension 
    $f_2:\R^{n-m}\to \R^{n}$, that is, $f_1|_{A'}=\tilde \phi$ and 
    $f_2|_{B'}=\tilde \psi$. Following  {\citep[Lemma 7.6]{madsen1997calculus}}, we define the maps
    $h_1:\R^n\times \R^{m-n}\to \R^n\times \R^{m-n}$,
    $$
    h_1(x,y)=(x,y+f_1(x))
    $$
    and $h_2:\R^n\times \R^{m-n}\to \R^n\times \R^{m-n}$,
    $$
    h_2(x,y)=(x+f_2(y),y).
    $$
    Observe that $h_2$ has the inverse map $h_2^{-1}(x,y)=(x-f_2(y),y)$. Then the map $$h=h_2^{-1}\circ h_1:\R^n\times \R^{m-n}\to \R^n\times \R^{m-n}$$ is a
    $C^k$-diffeomorphism that satisfies $h|_A=\phi.$ 
        This technique is called the `clean trick'. 
    
   Finally, to obtain the claim, we observe that when $R:\R^n\to \R^m$, $R(x)=(x,0)\in \{0\}^n\times \R^{m-n}$ is the zero padding operator, we have
    $$
    f(x)=\Psi^{-1}(\phi(R(x))),\quad x\in \R^n.
    $$
    As $h|_X=\phi$ and $R(x)\in X$, this yields
    $$
    f(x)=\Psi^{-1}(h(R(x))),\quad x\in \R^n,
    $$
    that is, $$f=E\circ R$$ where $E=\Psi^{-1}\circ h:\R^m\to \R^m$ is a $C^k$ diffeomorphism. Thus $f\in {\mathcal I}^k(\R^n,\R^m)$. This proves  \eqref{I-formula B} when $m\ge 3n+1$.}
\end{proof}

\subsection{Universality}
\label{sec:proofs:universality-results}

\subsubsection{The Proof of Lemma \ref{lem:mep-implied-from-univ}}
\label{sec:lem:mep-implied-from-univ}

\begin{proof}[The proof of Lemma \ref{lem:mep-implied-from-univ}]
    \begin{itemize}
        \item[(i)] 
        Let us consider $\epsilon>0$, a compact set $K\subset \R^n$ and  $f \in  \emb(\R^n,\R^m)$. Let  $W=K\times \{0\}^{o-n}$ and $F:\R^o\to \R^m$ be the map given by $F(x,y)=f(x)$, $(x,y)\in \R^{n}\times \R^{o-n}$. Because $\cR^{o,m}\subset \emb(\R^o,\R^m)$ is a uniform universal approximator of $C(\Rea^n,\Rea^m)$, there is an $R \in \cR^{o,m}$ such that $\norm{F - R}_{L^\infty(W)} < \epsilon$. Then for the map $E = I\circ R$ we have that $ B_{K,W}(f,E) < \epsilon$. This is true for every $\epsilon > 0$, and so $\cE^{o,m}$ has the MEP property w.r.t. the family $ \emb(\R^n,\R^m)$.
        \item[(ii)] Recall that $f \coloneqq \Phi_0 \circ R_0$ for $\Phi_0 \in \hbox{Diff}^1(\Rea^m,\Rea^m)$ and linear $R_0 \colon \Rea^n \to \Rea^m$, and that $R \in \cR$ is such that $\restr{R}{\overline U}$ is linear for open $U$. We present the proof in the case when $n = o$, and we make the assumption that $\restr{R}{K}$ is linear.  In this case, we have the existence of an affine map $A \colon \Rea^m \to \Rea^m$ so that $R_0=A\circ R$ so that 
        $\tilde K \coloneqq R_0(K) = A(R(K))$. Let $\epsilon > 0$ be given. By \citep[Chapter 2, Theorem 2.7]{hirsch2012differential},  the space $\hbox{Diff}^2(\R^m,\R^m)$ is dense in the space $\hbox{Diff}^1(\R^m,\R^m)$, and so there is some $\Phi_1 \in \hbox{Diff}^2(\Rea^m, \Rea^m)$ such that 
        $$
           \| \Phi_1|_{\tilde K}-\Phi_0|_{\tilde K}\|_{L^\infty({\tilde K};\R^m)}<\frac\epsilon2.
        $$
        Then, let $T \in \cT^m$ be such that $ \|T - \Phi_1\circ A\|_{L^\infty({R(K)};\R^m)}<\frac{\epsilon}{2}$. Then we have that
        \begin{align*}
            \norm{T\circ R - f}_{L^\infty(K)} &= \norm{T\circ R - \Phi_0 \circ R_0}_{L^\infty(K)}\\
            &\leq \norm{T\circ R - \Phi_1 \circ A \circ R}_{L^\infty(K)} + \norm{ \Phi_1 \circ A \circ R - \Phi_0 \circ R_0}_{L^\infty(K)}\\
            &\leq \norm{T - \Phi_1 \circ A }_{L^\infty(R(K))} + \norm{\Phi_1 \circ A \circ R - \Phi_0 \circ A \circ R}_{L^\infty(K)} \\
            &< \frac\epsilon2 + \frac\epsilon2 = \epsilon.
        \end{align*}
        Hence, if we let $r = T\circ R \circ f^{-1} \in \emb(f(K),T\circ R(K))$ then we obtain that $B_{K,K}(f,T\circ R) < \epsilon$. This holds for any $\epsilon$, and hence we have that $\cT\circ \cR$ has the MEP for $\cI(\Rea^n,\Rea^m)$.
        
        The proof in the case that $o \geq n$ follows with minor modification, and applying Lemma \ref{lem:b-k-helper} point 5.
\end{itemize}
\end{proof}

\subsubsection{The Proof of Example \ref{examp:lots-of-archs-have-mep}}
\label{sec:examp:lots-of-archs-have-mep}
\begin{proof}
    \begin{itemize}
        \item[(i)] From \citep[Theorem 15]{puthawala2020globally} we have that $\cR^{o,m}$ can approximate any continuous function $f \in \emb(\R^n,\R^m)$. Further, clearly (T1) and (T2) both contain the identity map, thus Lemma \ref{lem:mep-implied-from-univ} (i) applies.
        \item[(ii)]  Let $\cT^{m}$ be the family autoregressive flows with sigmoidal activations defined in \citep{huang2018neural}. By  \citep[App.\ G,\ Theorem 1 and Proposition 7]{teshima2020coupling}, $\cT^{m}$ are $\sup$-universal approximators in the space $\hbox{Diff}^2(\R^m,\R^m)$ of $C^2$-smooth diffeomorphisms $\Phi:\R^m\to \R^m$. When $\cR^{o,m}$ is one of (R1) or (R2) the network is always linear, hence the conditions are satisfied. If $\cR^{o,m}$ is (R4), then $\cR^{o,m}$ contains linear mappings, and if (R3), then we can shift the origin, so that $R(x)$ is linear on $K$. In all cases, Lemma \ref{lem:mep-implied-from-univ} part (ii) applies.
        
    \end{itemize}
\end{proof}

\subsubsection{The Proof of Theorem \ref{thm:univ-qual}}
\label{sec:proofs:thm:univ-qual}

\begin{proof}[The proof of Theorem \ref{thm:univ-qual}]

{ First we prove the claim under the assumptions (i).}

{ First we prove the claim under assumption (i).

Let $W\subset \R^n$ be an open relatively compact set.
 From Lemma \ref{lem:comp-mep} we have that 
    \begin{align}
        \cE^{n,m} \coloneqq{\cE_L^{n_{L-1}, m}\circ \dots \circ \cE_1^{n, n_{1}}}    \end{align}
    has the $m,n,n$ MEP w.r.t. {$\cF \coloneqq  \cF_{L}^{n_{L-1}, m}\circ\dots\circ\cF_{1}^{n, n_{1}}$}. Thus for any ${\epsilon_1} > 0$, we have an ${\tilde E} \in \cE^{n,m}$ s.t. $B_{K,W}(f,{\tilde E}) < \epsilon_1$. 
    
    From Lemma \ref{lem:b-k-helper} point 9, we have the existence of a $\mu' \in \cP(W)$ so that $        \wasstwo{\pushf{f}{\mu}}{\pushf{{\tilde E}}{\mu'}} <{\epsilon_1}$. 
    By convolving $\mu'$ with a suitable mollifier $\phi$, we can obtain a measure $\mu''=\mu'*\phi \in \cP(W)$ that is absolutely continuous with respect to the Lebesgue measure so that 
       $$\wasstwo{\mu'}{\mu''} <\frac{\epsilon_1}{1+\lip(\tilde E)},$$ see
       \citep[Lemma 7.1.10.]{ambrosio2008gradient}, and so 
        $\wasstwo{\pushf{{\tilde E}}{\mu'}}{\pushf{\tilde E}{\mu''}} <{\epsilon_1}$. Hence,
    \begin{align}
        \wasstwo{\pushf{f}{\mu}}{\pushf{{\tilde E}}{\mu''}} < 2\epsilon_1.
    \end{align}
    Next, from universality of $\cT^n_0$ for any $\epsilon_2 > 0$, we have the existence of a $T_0 \in \cT^n_0$ so that $\wasstwo{\mu''}{\pushf{T_0}{\mu}} < \epsilon_2$. From Lemma \ref{lem:b-k-helper} points 7 and 8 we have that 
    \begin{align}
        \wasstwo{\pushf{f}{\mu}}{\pushf{{\tilde E}\circ T_0}{\mu}} \leq 2\epsilon_1 + \epsilon_2 \lip({\tilde E}).
    \end{align}
    For a given $\epsilon > 0$, choosing $\epsilon_1 < \frac\epsilon{{4}}$ and $\epsilon_2 < \frac{\epsilon}{2(1+\lip({\tilde E}))}$ yields that the map ${E}= {\tilde E} \circ T_0  \in {\cE}$  is such that $\wasstwo{\pushf{f}{\mu}}{\pushf{{E}}{\mu}} < \epsilon$. This yields the result.
 
    Next we prove the claim under the assumptions (ii).
    By our assumptions, in the weak topology of the space $C^2(\R^{n_{j}},\R^{n_{j}})$, the closure of the set
    $\cT^{n_{j}}\subset C^2(\R^{n_{j}},\R^{n_{j}})$ 
    contains the space of $\hbox{Diff}^2(\R^{n_{j}},\R^{n_{j}})$.
    Moreover, by our assumptions $\cR^{n_{j-1}, n_{j}}$ contains a linear map $R$. We observe that 
    as $\cR^{n_{j-1}, n_{j}}$ is a space of expansive elements, the map $R$ is injective.
    and hence by Lemma \ref{lem:mep-implied-from-univ},
    the family $$\cE_j^{n_{j-1},n_{j}}=\cT^{n_{j}}\circ \cR^{n_{j-1},n_{j}}$$
    has the MEP w.r.t. $\cF = \cI^1(\Rea^n,\Rea^m)$. By Theorem \ref{thm:emb-and-extens-coincide}, we have that $\cI^1(\Rea^n,\Rea^m)$ coincides with the space  $\emb^1(\Rea^n, \Rea^m)$. Finally, by  the assumption that  $\cT^{n_0}_0$ is dense in the space of $C^2$-diffeomorphism $\hbox{Diff}^2(\R^{n_\ell})$ implies that  $\cT^{n_0}_0$ is a $L^p$-universal approximator for the set of $C^\infty$-smooth triangular maps for all $p<\infty$. Hence by  Lemma 3 in Appendix A of \citep{teshima2020coupling},  $\cT^{n_0}_0$ is a distributionally universal. From these the claim in the case (ii) follows in the same way as the case (i) using the family $\cF= \emb^1(\Rea^n, \Rea^m)$.}
\end{proof}

\subsubsection{The Proof of Lemma \ref{lem:mep-is-nesc-for-universality}}

\label{sec:proofs:mep-is-nesc-for-universality}
\begin{proof}[The proof of Lemma \ref{lem:mep-is-nesc-for-universality}]
    The proof follows from taking the logical negation of the MEP for $\cF$. If the MEP is not satisfied, then there is some $f \in \cF$ so that $B_{K,W}(f,E)$ is never smaller than $\epsilon > 0$ for all $E \in \cE$. Applying the definition of $B_{K,W}(f,E)$ from Eqn. \ref{eqn:b-k-def} yields the result.
\end{proof}

\subsubsection{The Proof of Cor. \ref{cor:decoupling}}
\label{sec:cor:decoupling}

\begin{proof}[The proof of Cor. \ref{cor:decoupling}]
    {The proof of Eqn \ref{eqn:decoupling:later-layers-min:1} follows from the definition of the MEP. }
    
    From Eqn. \ref{eqn:decoupling:later-layers-min:1} for $i = 1,\dots $ we have the existence of a $\epsilon_i \coloneqq B_{K,W}(f,E_i)$, where $\lim_{i \to \infty} \epsilon_i = 0$, and a $r_i \in \emb(f(K),E_i(W))$ such that $\norm{I - r_i}_{L^\infty(f(K))} \leq 2 \epsilon_i$. Applying Lemma \ref{lem:b-k-helper} point 8, we have that for any $E'\in \cE^{n,o}(X,W)$
    \begin{align}
        B_{K,X}(f,E_i\circ E') \leq 2 \epsilon_i + \lip(E_i)B_{K,X}(E^{-1}_i\circ r_i \circ f,E').
    \end{align}
    Because $\cE^{n,o}(X,W)$ has the $o,n,n$ MEP, for each $i = 1,\dots$, we can find a $E'_i \in \cE^{n,o}(X,W)$ such that $B_{K,X}(E^{-1}_i\circ r_i \circ f,E'_i) \leq \frac{1}{1 + \lip(E_i)} \epsilon_i$, and so $B_{K,X}(f,E_i\circ E'_i) \leq 3\epsilon_i$. For this choice of $E_i'$, we have that $\lim_{i \to \infty} B_{K,X}(f,E_i\circ E_i') = 0$. 
    
    From Lemma \ref{lem:b-k-helper} point 9, we have that for any absolutely continuous $\mu \in \cP(K)$, there is a absolutely continuous $\mu' \in \cP(X)$ such that $\wasstwo{\pushf{f}{\mu}}{\pushf{E_i \circ E'_i}{\mu'}} \leq 3\epsilon$. By the universality of $\cT^n$, continuity of $E_i \circ E'_i$, and absolute continuity of $\mu$ and $\mu'$, we have the existence of $T_i \in \cT^n$ so that 
    \begin{align}
        \wasstwo{\pushf{f}{\mu}}{\pushf{E_i \circ E'_i\circ T_i}{\mu}} \leq 4\epsilon_i
    \end{align}
    for each $i = 1,\dots$. This proves the claim.
\end{proof}

\subsubsection{Further Discussion on Matching Topology Exactly vs Approximately}
\label{sec:approximate-vs-exact-matching}

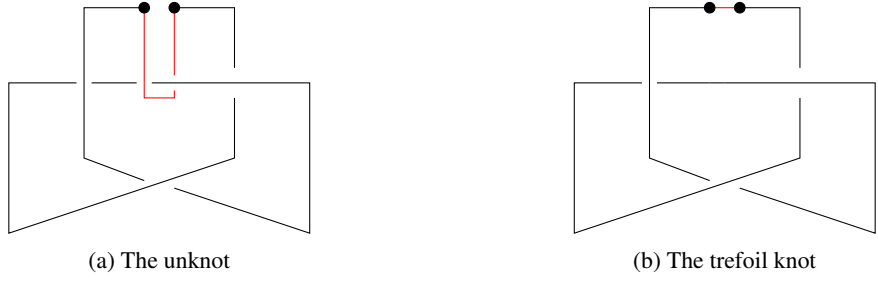
\begin{figure}
    \centering
    \begin{subfigure}{0.45\linewidth}
        \centering
        \begin{tikzpicture}
        \draw[red] (-.2,1)--(-.2,-.2)--(.2,-.2)--(.2,-.1);
        \draw[red] (.2,.1)--(.2,1);
        \filldraw[black] (-.2,1) circle (2pt);
        \filldraw[black]  (.2,1) circle (2pt);
        \draw[black] (.2,1)--(1,1)--(1,.2);
        \draw[black] (1,-.2)--(1,-1)--(-2,-2)--(-2,0)--(-1.1,0);
        \draw[black] (-.9,0)--(-.3,0);
        \draw[black] (-.1,0)--(2,0)--(2,-2)--(.2,-1.4);
        \draw[black] (-.2,-1.3)--(-1,-1)--(-1,1)--(-.2,1);
        \end{tikzpicture}
        \subcaption{The unknot}
        \label{fig:unknot-to-trefoil-knot:unknot}
    \end{subfigure}
    \begin{subfigure}{0.45\linewidth}
        \centering
        \begin{tikzpicture}
        \draw[red] (.2,1)--(-.2,1);
        \filldraw[black] (-.2,1) circle (2pt);
        \filldraw[black]  (.2,1) circle (2pt);
        \draw[black] (0,0)--(-.2,0);
        \draw[black] (.2,1)--(1,1)--(1,.2);
        \draw[black] (1,-.2)--(1,-1)--(-2,-2)--(-2,0)--(-1.1,0);
        \draw[black] (-.9,0)--(-.2,0);
        \draw[black] (0,0)--(2,0)--(2,-2)--(.2,-1.4);
        \draw[black] (-.2,-1.3)--(-1,-1)--(-1,1)--(-.2,1);
        \end{tikzpicture}
        \subcaption{The trefoil knot}
        \label{fig:unknot-to-trefoil-knot:trefoil-knot}
    \end{subfigure}
    \caption{{An example showing how the unknot (left) can be deformed to approximate the trefoil knot (right). The \hlblack{black} part of both knots are identical, and the \hlred{red} section can be made arbitrarily skinny by bringing the \hlblack{black} points together. This can be done while sending the measure of the \hlred{red} sections to zero, if the starting measure have no atoms. In this way, we can construct a sequence of diffeomorphisms $\paren{E_i}_{i=1,\dots}$ so that $\wasstwo{\pushf{E_i}{\mu}}{\nu} \to 0$ where $\mu$ is the uniform measure on $S^1$, and $\nu$ the uniform measure on the trefoil knot. 
    We would like to thank Reviewer 4 for suggesting this discussion and providing the figure (in tikz code!).}}
    \label{fig:unknot-to-trefoil-knot}
\end{figure}

{

In this section we discuss a theoretical gap between the positive approximation results of Theorem \ref{thm:univ-qual} and the negative exact mapping results of Lemma \ref{lem:mep-is-nesc-for-universality}. We show two main results. 

First we construct sequences of maps of the form $\cE = \cT \circ \cR$ that map the uniform measure on $S^1$ to the uniform measure on the trefoil knot. As discussed in Section \ref{sec:top-obstructions}, there are no mappings of this form which map $S^1$ to the trefoil knot exactly, but there are approximate mappings. This shows that there is some overlap between the two results, and extendable mappings may be approximated by non-extendable mappings.

Second we prove that sequences of functions that approximate non-extendable embeddings with extendable ones necessarily have unbounded gradients. This result shows that, when restricted to approximation by sequences with bounded gradients, either Theorem \ref{thm:univ-qual} or Lemma \ref{lem:mep-is-nesc-for-universality} can apply, but never both.


\begin{example}
    There is a sequence of extendable embeddings $\paren{E_i}_{i=1,\dots}$ that map the uniform measure on $S^1$, denoted $\mu$, to the uniform measure on the trefoil knot, denoted $\nu$, so that 
    \begin{align*}
        \lim_{i\to\infty}\wasstwo{\pushf{E_i}{\mu}}{\nu} = 0.
    \end{align*}
\end{example}

\begin{proof}
    The key idea of the construction is shown in Figure \ref{fig:unknot-to-trefoil-knot}. In that figure the unknot is bent so that it overlaps the trefoil knot, outside of an exceptional set (shown in \hlred{red} in Figure \ref{fig:unknot-to-trefoil-knot}) which can be made as small as desired. The result follows by constructing a sequence of functions which `squeeze' this red section as small as possible.

    Let $\mu$ be the uniform probability measure on $S^1\subset \R^2$, and $\nu$ the uniform probability measure on the trefoil knot, $\cM$. Let $R\colon \R^2 \to \R^3$ be a fixed linear map of the form $R(x) = (x,0)$. 
    
    We define a sequence $\paren{X_i}_{i=1,\dots}$ of unknots in the following way. For any choice of two points on the top of the trefoil knot as shown in \hlblack{black} in Figure \ref{fig:unknot-to-trefoil-knot:unknot}, we can replace the straight-line \hlred{red} section with a U-shaped section as shown in Figure \ref{fig:unknot-to-trefoil-knot:unknot} so that the resulting knot is the unknot. We obtain $X_1$ by letting the \hlblack{black} points be a distance 1 apart, $X_2$ by letting them be a distance $\frac12$ apart and so on, so that for $X_i$ the two points are a distance $\frac1i$ apart. Further, for each $X_i$, we define $A_i$ and $B_i$ where $A_i$ is the U-shaped piece of $X_i$ (in \hlred{red}), and $B_i = X_i \setminus A_i$. Observe that $B_i \subset \cM$.
    
    Let $\paren{T_i'}_{i=1,\dots}$ be a family of diffeomorphisms so that $E_i\colon \R^3 \to \R^3$ maps $S^1\times \set{0}$ to $X_i$. Further, let $\paren{T''_i}_{i=1,\dots}$ be such that $T''_i\colon X_i \to X_i$ so that $\chi_{B_i}\paren{T''_i\circ T'_i\circ R}_\#{\mu} = \chi_{B_i}\nu$ when $\chi_{B_i}$ is the characteristic function of the set $B_i$.
    
    Then we define $E_i \coloneqq T''_i\circ T'_i\circ R$ and compute
    \begin{align*}
        \wasstwo{E_i{}_\#{\mu}}{\nu} &\leq \wasstwo{\chi_{A_i}E_i{}_\#{\mu}}{\chi_{\cM \setminus B_i}\nu} + \wasstwo{\chi_{B_i}E_i{}_\#{\mu}}{\chi_{B_i}\nu}\\
        &= \wasstwo{\chi_{A_i}E_i{}_\#{\mu}}{\chi_{\cM \setminus B_i}\nu}.
    \end{align*}
    As $i$ increases, the length of $\cM \setminus B_i$ goes to zero, thus $ \nu(\cM\setminus B_i) = \mu(A_i)$ converges to zero. Hence taking limits yields 
    \begin{align*}
        \lim_{i \to \infty} \wasstwo{E_i{}_\#{\mu}}{\nu} \leq \lim_{i \to \infty} \wasstwo{\chi_{A_i}E_i{}_\#{\mu}}{\chi_{\cM \setminus B_i}\nu} = 0.
    \end{align*}
    Finally, $E_i$ is certainly an extendable embedding, as $R$ is linear, and $T''_i\circ T'_i$ are diffeomorphisms.
\end{proof}

The above proof also applies when $\nu$ or $\mu$ have finitely many atoms. The same construction works if $A_i$ is chosen so that it contains no atoms for sufficiently large $i$.

Next, we show that all function sequence for which implication of Theorem \ref{thm:univ-qual} and conditions of Lemma \ref{lem:mep-is-nesc-for-universality} apply are not uniformly Lipschitz. This implies that if they are differentiable they have unbounded gradients.

\begin{lemma}
    Let $f$ be continuous and $E_i$ be a sequence of continuous functions that are uniformly Lipschitz with constant $L$. Let $E_i$ be such that for all compact $K$ and $W$ subsets of $\R^n$, there is an $\epsilon > 0$, so $\forall i$ and $r \in \emb(f(K), E(W))$, $\norm{I - r}_{L^\infty(K)} \geq \epsilon$. If $\mu$ is the indicator function of $d$, then $\lim_{i \to \infty}\wasstwo{f_\#\mu}{E_i{}_\#\mu} > 0$.
\end{lemma}

\begin{proof}
    Let $E_i$ be uniformly Lipschitz with constant $L$. Consider a $\frac\epsilon2$ tubular neighborhood of $f(K)$. From the fact that $\norm{I - r}_{L^\infty(K)} \geq \epsilon$, we have that there is a point $x \in E(W)$ so that $x$ lies outside of this neighborhood. From uniform Lipschitzness of $E_i$, for each $i$ there is a ball $B$ of radius $\frac \epsilon{4L}$ around $x$ so that all points in $E_i\cap B$ are more than $\frac\epsilon4$ away from $f(K)$. We also have that $\mu(E_i\cap B) > c$ where $c$ is the volume of the $n$ dimensional ball. Thus, $\wasstwo{f_\#\mu}{E_i{}_\#\mu} > \frac{c\epsilon}{4L}$ for each $i$, and so $\lim_{i \to \infty} \wasstwo{f_\#\mu}{E_i{}_\#\mu} > 0$.
\end{proof}


}

\subsection{Layerwise Inversion and Recovery of Weights}
\label{sec:additional-prop-details}

\subsubsection{Layer-wise Projection}
\label{sec:additional-prop-details:layer-wise-projection}

Here we provide the details of our closed-form layerwise projection algorithm The flow layers are injective, and are often implemented to be numerically easy to invert. Thus, the crux of the algorithm comes from inverting the injective expansive layers, $R$. The range of the $\relu$ layer is piece-wise affine, hence the inversion follows a two-step program. First, identify which affine piece (described algebraically, onto which sign pattern) to project. Second, project to this point using a standard least-squares solver.

The second step is always straight-forward to analyze, but the first is more complicated. 

The key step in our algorithm is the fact that for the specific choice of weight matrix $W = \bmat{B\\-DB}$, given any $y \in \Rea^{2n}$, we can always solve the least-squares inversion problem exactly.

We prove this result in several parts given below.
\begin{enumerate}
    \item For any $y \in \Rea^{2n}$, $M_yW \in \Rea^{n\times n}$ is full-rank.
    \item If $\bracketed{y}_{i} \neq \bracketed{y}_{i + n}$ for each $i = 1,\dots,n$, then the $\argmin$ in Eqn. \ref{eqn:thm:proj-lemma} is well defined, i.e. that there is a unique minimizer. Otherwise there are $2^I$ minimizers, where $I$ is the number of distinct $i$ such that $\bracketed{y}_i = \bracketed{y}_{i+n}$.
    \item If $\tilde M_y = \bmat{\Delta_y&(I^{n\times n} - \Delta_y)}$, then
    \begin{align}
        \min_{x \in \Rea^n} \norm{y - R(x)}^2_2 = \min_{x \in \Rea^n}\norm{M_y\paren{y - Wx}}_2^2 + \norm{\tilde M_y y}^2_2.
    \end{align}
    \item We verify Eqn. \ref{eqn:thm:proj-lemma}.
\end{enumerate}

\begin{proof}[The proof of Theorem \ref{thm:proj-lemma}]
    \begin{enumerate}
        \item Using the definition of $M_y$, we have,
            \begin{align}
                M_y\bmat{B \\ -DB} = \paren{I^{n\times n} -\Delta_{y}} B -\Delta_{y}DB = \paren{I^{n\times n} -\Delta_{y} - \Delta_{y}D}B.
            \end{align}
            But, $\paren{I^{n\times n} -\Delta_{y} - \Delta_{y}D}$ is a full-rank diagonal matrix (with entries either $1$ or $[D]_{i,i}$), and $B$ is full rank by assumption, hence $M_y\bmat{B \\ -DB}$ is too.
        \item 
            Because $B$ is square and full rank there exists a basis\footnote{Namely the columns of the matrix $B^{-1}$} $\set{\hat b_i}_{i = 1,\dots,n}$ of $\Rea^n$ such that \begin{align}
                \innerprod{\hat b_j}{b_i} = \begin{cases} 
                    1 &\text{if } i = j\\
                    0 &\text{if } i \neq j
                \end{cases}.
            \end{align} 
            For an $x \in \Rea^n$, let $\alpha_i = \innerprod{x}{b_i}$ for $i = 1,\dots, n$ be the expansion of $x$ in the $\hat b_i$ basis. 
            \begin{align}
                \label{eqn:proj:step-0.25}
                \min_{x \in \Rea^n}\norm{y - R(x)}_2^2 &= \min_{x \in \Rea^n}\sum_{i = 1}^{2n} \bracketed{y - R(x)}_i^2\\
                \label{eqn:proj:step-0.5}
                &=\sum^n_{i = 1} \min_{x_i \in \Rea} \paren{\bracketed{y}_i - \max(\innerprod{x}{b_i},0)}^2 + \paren{\bracketed{y}_{i+n} - \max(\innerprod{x}{-\bracketed{D}_{ii}b_i},0)}^2
            \end{align}
            We now consider minizing Eqn. \ref{eqn:proj:step-0.5} by minimizing the basis expansion in terms of $\alpha_i$,
            \begin{align}
                \label{eqn:proj:step-1}
                \sum^n_{i = 1} \min_{\alpha_i \in \Rea} \paren{\bracketed{y}_i - \max(\alpha_i,0)}^2 + \paren{\bracketed{y}_{i+n} - \max(-\bracketed{D}_{ii}\alpha_i,0)}^2
            \end{align}
            Eqn. \ref{eqn:proj:step-1} is clearly minimized by minizing each term in the sum, hence we search for a minimizer of the $i$'th term
            \begin{align}
                \label{eqn:proj:step-2}
                \min_{\alpha_i \in \Rea} \paren{\bracketed{y}_i - \max(\alpha_i,0)}^2 + \paren{\bracketed{y}_{i+n} - \max(-\bracketed{D}_{ii}\alpha_i,0)}^2
            \end{align}
            Noting $f(\alpha_i)$ as the quantity inside the minimum of Eqn. \ref{eqn:proj:step-2}, we consider the positive, negative and zero $\alpha_i$ cases of Eqn. \ref{eqn:proj:step-2} separately and we get
            \begin{align}
                \min_{\alpha_i \in \Rea^+} f(\alpha_i) = \min_{\alpha_i \in \Rea^+} \paren{\bracketed{y}_i - \alpha_i}^2 + \bracketed{y}_{i+n}^2 &= \bracketed{y}_{i+n}^2\\ 
                \label{eqn:proj:step-2.5}
                \min_{\alpha_i \in \Rea^-} f(\alpha_i) = \min_{\alpha_i \in \Rea^+} \bracketed{y}_i^2 + \paren{\bracketed{y}_{i+n} + \bracketed{D}_{ii}\alpha_i}^2 &= \bracketed{y}_{i}^2\\
                f(0) &= \bracketed{y}^2_{i} + \bracketed{y}^2_{i+n}.
            \end{align}
            If $\bracketed{y}_{i+n} > \bracketed{y}_{i}$, then the minimizer of Eqn. \ref{eqn:proj:step-2} is $\alpha_i = - \frac{\bracketed{y}_{i+n}^2}{\bracketed{D}_{ii}} < 0$. Conversely if $\bracketed{y}_{i+n} < \bracketed{y}_{i}$ then the minimizer of Eqn. \ref{eqn:proj:step-2} is $\alpha_i = \bracketed{y}_i > 0$. This argument applies all $i = 1,\dots,n$, and hence if $\bracketed{y}_i \neq \bracketed{y}_{i+1}$ for all $i = 1,\dots,n$ then the minimizing $x$ is unique.
            
            If $\bracketed{y}_i = \bracketed{y}_{i+1}$ then there are exactly two minimizers of $f(\alpha_i)$, $- \frac{\bracketed{y}_{i+n}^2}{\bracketed{D}_{ii}}$ and $\bracketed{y}_i$, for both of which $f(\alpha_i) = \bracketed{y}^2_i = \bracketed{y}^2_{i + n}$.
        \item If we suppose that $\bracketed{y}_{i+n} - \bracketed{y}_i > 0$, then $\bracketed{c(y)}_i = 0$ and $\bracketed{c(y)}_{i+n} > 0$, thus $\bracketed{\Delta_y}_{ii} = 1$, hence if we let $x_{\min}$ be the minimizing $x$ from part 1, then
            \begin{align}
                &\paren{\bracketed{y}_i - \max(\innerprod{x_{\min}}{b_i},0)}^2 + \paren{\bracketed{y}_{i+n} - \max(\innerprod{x_{\min}}{-\bracketed{D}_{ii}b_i},0)}^2\\
                &= \bracketed{y}_i^2 + \paren{\bracketed{y}_{i+n} - \max(\innerprod{x_{\min}}{-\bracketed{D}_{ii}b_i},0)}^2\\
                \label{eqn:proj:step-4}
                &= \bracketed{\tilde M_y y}_i^2 + \bracketed{M_y \paren{y - W x_{\min}}}^2_i
            \end{align}
            If $\bracketed{y}_{i+n} - \bracketed{y}_i \leq 0$ then we have 
            \begin{align}
                &\paren{\bracketed{y}_i - \max(\innerprod{x_{\min}}{b_i},0)}^2 + \paren{\bracketed{y}_{i+n} - \max(\innerprod{x_{\min}}{-\bracketed{D}_{ii}b_i},0)}^2\\
                &=\paren{\bracketed{y}_i - \max(\innerprod{x_{\min}}{b_i},0)}^2 + \bracketed{y}_{i+n}^2\\ 
                \label{eqn:proj:step-5}
                &= \bracketed{M_y \paren{y - W x_{\min}}}^2_i + \bracketed{\tilde M_y y}_i^2.
            \end{align}
            Thus combining Eqn.s \ref{eqn:proj:step-0.25}, \ref{eqn:proj:step-0.5}, \ref{eqn:proj:step-4} and \ref{eqn:proj:step-5} for each $i = 1,\dots,n$, we have that 
            \begin{align}
                \min_{x \in \Rea^n} \norm{y - R(x)}^2_2 = \min_{x \in \Rea^n} \norm{M_y\paren{y - Wx}}^2_2 + \norm{M_y y}^2_2.
            \end{align}
        \item For the final point, combining all of the above points we have
            \begin{align}
                \label{eqn:proj:step-9}
                \min_{x \in \Rea^n} \norm{y - R(x)}^2_2 = \min_{x \in \Rea^n} \norm{M_y\paren{y - Wx}}^2_2.
            \end{align}
            Further we have from Point 1 that $M_yW$ is full rank, hence $\paren{M_y W}^{-1} M_y y = R^{\dagger}(y)$ is a minimizer of Eqn. \ref{eqn:proj:step-9}. If $\bracketed{y}_{i} \neq \bracketed{y}_{i+n}$ for all $i = 1,\dots,n$ then Part 2 applies, and $R^\dagger(y)$ is the unique minimizer of $\norm{y - R(x)}^2_2$. In either case, we have that $R^\dagger(y)$ is a minimizer.
    \end{enumerate}
\end{proof}

\subsubsection{Black-box recovery}
\label{sec:additional-prop-details:black-box-recovery}

We now discuss assumptions that enable black-box recovery of the weights of our entire network post-training. 

\begin{assumption}
    \label{asmp:black-box-recovery}
    For each $\ell = 1,\dots,L$, $\cR_\ell$ is an affine $\relu$ layer. Each $\cT_\ell$ and $\cT_0$ is constructed from a finite number of affine $\relu$ layers. 
\end{assumption}

\begin{remark}
    \label{rmk:black-box-recovery}
    If a network $\cF$ of the form of Eqn. \ref{eqn:network-def} satisfies Assumption \ref{asmp:black-box-recovery}, then given the range of the network, the {range of the network can be recovered exactly.}
    
    {Further, if the linear region assumption from \citep{rolnick2020reverse} is satisfied, then the exact weights are recovered, subject to two natural isometries discussed below.}
\end{remark}

\begin{remark}
    The $\relu$ part of Assumption \ref{asmp:black-box-recovery} is for all examples in Sec. \ref{sec:expans}. Further it is also satisfied by both flows considered in Sec. \ref{sec:expres}, provided that the various $g_i$ are given by layers of affine $\relu$'s.
\end{remark}

In \citep{rolnick2020reverse}, the authors show that, although $\relu$ networks depend on the value of their weight matrix in non-linear ways, it is still possible to recover the exact weights of a given $\relu$ network in a black-box way, subject to natural isometrics. The authors show that this is possible not only in theory, but in numerical applications as well.

The works of \citep{rolnick2020reverse,bui2020functional} imply that provided the activation functions of the expressive elements are $\relu$ then the entire network can be recovered in a black-box way. Further, provided that either the `linear region assumption' from \citep{rolnick2020reverse} or the generality assumption from \citep{bui2020functional} is satisfied, then the entire network can be recovered \emph{uniquely} modulo the natural isometries of rescaling and permutation of weight matrices.

First we describe the two natural isometries of scaling and permutation. Consider the following function
\begin{align}
    f(x) = W_2 \phi(W_1x)
\end{align}
where $\phi$ is coordinate-wise homogeneous degree 1 (such as $\relu$) and $W_1 \in \Rea^{n_1\times n_2}$ and $W_2 \in \Rea^{n_2 \times n_3}$. If we let $P \in \Rea^{n_2\times n_2}$ be any permutation matrix, and $D_+$ be
a diagonal matrix with strictly positive elements, then we can write
\begin{align}
    f(x) = W_2 P' D_+^{-1} \phi( D_+ P W_1x)
\end{align}
as well. Thus $\relu$ networks can only ever be uniquely given subject to these two isometries. When describe unique recovery in the rest of this section, we mean modulo these two isometries.

In \citep{rolnick2020reverse}, the authors describe how all parameters of a $\relu$ network can be recovered \emph{uniquely} (called reverse engineered in \citep{rolnick2020reverse}), subject to the so called `linear\footnote{The use of `linear' in this context is somewhat non-standard, and instead means  affine. In this section we use the term `linear region assumption', but use `affine' where \citep{rolnick2020reverse} would use `linear' to preserve mathematical meaning.} region assumption', LRA.

The input space $\Rea^n$ can be partitioned into a finite number of open $\set{\cS_i}^{n_i}_{i = 1}$, where for each $k$, $f(x) = \bW_k i  + \bb_i$, i.e. the network corresponds to an affine polyhedron in the output space. The algorithms \citep[Alg.s 1 \& 2]{rolnick2020reverse} are roughly described below.

First, identify at least one point within each affine polyhedra $\set{\cH_j}_{j = 1}^{n_j}$. Then identify the boundaries between polyhedra. The boundaries between sections are always one affine `piece' of piecewise hyperplanes $\set{\cH_{j}}_{j = 1}^{n_j}$. These $\set{\cH_{j}}_{j = 1}^{n_j}$ are the central objects which indicate the (de)activation of an element of a $\relu$ somewhere in the network. If the $\cH_{j}$ are full hyperplanes, then the $\relu$ that is (de)activates occurs in the first layer of the network. If $\cH_{j}$ is not a full hyperplane, then it necessarily has a bend where it intersects another hyperplane $\cH_{j'}$. Further, except for a Lebesgue measure 0 set, when $\cH_{j}$ intersects $\cH_{j'}$ the latter does not have a bend. If this is the case, then $\cH_{j'}$ corresponds to a $\relu$ (de)activation in an earlier layer than $\cH_j$. In this way the activation functions of every layer can be deduced. Once this is done, the normals of the hyperplanes can be used to infer the row-vectors of the various weight matrices, letting one recover the entire network.


The above algorithm recovers all of the weights exactly provided that the LRA is satisfied. The LRA is satisfied if for every distinct $\cS_i$ and $\cS_{i'}$, either $\bW_{i} \neq \bW_{i'}$ or $\bb_{i} \neq \bb_{i'}$. That is, different sign patterns produce different affine sections in the output space. This is a natural assumption, as the algorithm as described above reconstruction works by first detecting the boundaries between adjacent affine polyhedra, which is only possible if the LRA holds.

Given the weights of a network there is currently no simple way to detect if the LRA is satisfied, to our knowledge. Nevertheless the authors of \citep{rolnick2020reverse} show that if it is satisfied, then unique recovery follows. Nevertheless recovery of the {range of the} entire network is possible, {but this recovery may not be unique.}

In \citep{bui2020functional} the authors also consider the problem of recovering weights of a $\relu$ neural network, however the authors therein study the question of \emph{when} there exist isometries beyond the two natural ones described above. In particular the main result \citep[Theorem 1]{bui2020functional} shows the following. Let $\cE^{n_0,n_L}$ be a $\relu$ network that is $L$ layers deep and non-increasing. Suppose that $E_1, E_2 \in \cE^{n_0,n_L}$, $E_1$ and $E_2$ are general\footnote{A set is general in the topological sense if its complement is closed and nowhere dense} and for all $x \in \Rea^{n_0}$ $E_1(x) = E_2(x)$, then $E_1$ is parametrically identical to $E_2$ subject to the two natural isometries.

This work provides the stronger result, however does not apply to the networks that we consider out of the box. It does apply to our expressive elements (provided that they use $\relu$ activation functions, and are non-increasing), but not necessarily apply to the network on the whole. 

\end{document}

%% file: common_packages.tex
\usepackage{amssymb}
\usepackage{amsmath}
\usepackage{graphicx}
\usepackage{lineno}
\usepackage{subcaption}
\usepackage[utf8]{inputenc}
\usepackage{float}
\usepackage{multicol}
\usepackage{lipsum}
\usepackage{marvosym}
\usepackage{mathtools}
\usepackage{soul}
\setstcolor{red}
\usepackage{bm}
\usepackage{amsthm}
\usepackage{xargs}                      
\usepackage{xcolor}
\usepackage{bbm}
\usepackage{hyperref}
\usepackage[title]{appendix}

%% file: common_commands.tex
\usepackage{mathtools}

\newcommand{\argmin}{\operatornamewithlimits{argmin}}

\newcommand{\Rea}{\mathbb{R}}

\newcommand{\R}{\mathbb{R}}

\newcommand{\set}[1]{\left \{ #1\right \}}
\newcommand{\norm}[1]{\left\lVert#1\right\rVert}

\newcommand{\innerprod}[2]{\left \langle#1,#2\right\rangle}
\DeclareMathOperator{\rank}{rank}

\newcommand{\Nat}{\mathbb{N}}

\newcommand{\wasstwo}[2]{\wasstwoop \left( #1, #2\right)}
\newcommand{\wasstwoop}{\operatorname{W}_2}

\newcommand{\relu}{\operatorname{ReLU}}

\newcommand{\bmat}[1]{\begin{bmatrix} #1 \end{bmatrix}}
\newcommand{\paren}[1]{\left ( #1\right)}
\newcommand{\bracketed}[1]{\left [ #1\right]}

\newcommand{\pushf}[2]{{#1}{}_{\#}{#2}}

\makeatletter
\def\mathcolor#1#{\@mathcolor{#1}}
\def\@mathcolor#1#2#3{%
  \protect\leavevmode
  \begingroup
    \color#1{#2}#3%
  \endgroup
}
\makeatother

\renewcommand{\eqref}[1]{{Eq.~\ref{#1}}}

\DeclareMathOperator*{\esssup}{ess\,sup}

\DeclareMathOperator*{\range}{Range}

\newtheorem{assumption}{Assumption}

\newcommand\restr[2]{{
  \left.\kern-\nulldelimiterspace 
  #1 
  \vphantom{
  \big|
  } 
  \right|_{#2} 
  }}
  
\makeatletter
\@ifclassloaded{beamer}{}{
    \newtheorem{lemma}{Lemma}
    \newtheorem{theorem}{Theorem}
    \newtheorem{corollary}{Corollary}

    \theoremstyle{definition}
    \newtheorem{definition}{Definition}
    \newtheorem{example}{Example}
    
    \theoremstyle{remark}
    \newtheorem{remark}{Remark}
}
\makeatother

\def\bszero{\boldsymbol{0}}
\DeclareMathOperator{\GL}{GL}

%% file: common_colors.tex
\definecolor{orange}{RGB}{255,127,0}
\definecolor{blackchocolate}{HTML}{191102}
\definecolor{rosewood}{HTML}{510d0a}
\definecolor{rufous}{HTML}{A31B14}
\definecolor{citron}{HTML}{a29f15}
\definecolor{orangeyellow}{HTML}{f3b61f}
\definecolor{teagreen}{HTML}{bbd8b3}
\definecolor{gray}{RGB}{128,128,128}
\definecolor{lightcitron}{RGB}{189,187,91}
\definecolor{lightrufous}{RGB}{190,95,90}
\definecolor{steelblue}{HTML}{4C86A8}

\newcommand{\hlblack}[1]{{\color{blackchocolate}#1}}
\newcommand{\hlred}[1]{{\color{rufous}#1}}

\newcommand{\hlgray}[1]{{\color{gray}#1}}
\newcommand{\hlorange}[1]{{\color{orange}#1}}